\documentclass[11pt,letterpaper]{article}
\usepackage[margin=1in]{geometry}
\usepackage{times}
\usepackage[authoryear,round]{natbib}
\setcitestyle{citesep={;},aysep={,},yysep={;}}
\usepackage{microtype}
\usepackage[T1]{fontenc}
\usepackage[utf8]{inputenc}
\usepackage{amsmath,amssymb,amsfonts,amsthm}
\usepackage{booktabs}
\usepackage{graphicx}
\graphicspath{{Arxiv-Template/figures/}{figures/}}
\usepackage{float}
\usepackage{algorithm}
\usepackage{algpseudocode}
\usepackage{tikz}
\usepackage{enumitem}
\usetikzlibrary{arrows.meta,positioning}
\usepackage[colorlinks=true,citecolor=blue,linkcolor=blue,urlcolor=blue]{hyperref}
\usepackage{url}
\usepackage{placeins}

\DeclareMathOperator*{\argmax}{arg\,max}
\newcommand{\Prob}{\mathbb{P}}
\newcommand{\KL}{\mathrm{D}_{\mathrm{KL}}}
\newcommand{\Alt}{\mathrm{Alt}}
\newcommand{\Var}{\mathrm{Var}}
\newcommand{\Cov}{\mathrm{Cov}}
\newcommand{\Cat}{\mathrm{Cat}}
\newcommand{\Dir}{\mathrm{Dir}}
\newcommand{\Kinf}{\mathcal{K}_{\inf}}

\theoremstyle{plain}
\newtheorem{theorem}{Theorem}
\numberwithin{theorem}{section}

\newtheorem{assumption}{Assumption}
\newtheorem{prop}{Proposition}[section]

\title{Adaptive Self-Consistency: From Black-Box Sampling to Distribution-Valued Feedback}
\author{%
  \begin{tabular}[t]{@{}c@{}}
    Jingkai Huang\thanks{Equal contribution.}\\
    {\small Stern School of Business}\\
    {\small New York University}\\
    {\small\href{mailto:jh9959@nyu.edu}{\textcolor{black}{\nolinkurl{jh9959@nyu.edu}}}}
  \end{tabular}\qquad
  \begin{tabular}[t]{@{}c@{}}
    Yunfan Zhang\footnotemark[1]\\
    {\small Stern School of Business}\\
    {\small New York University}\\
    {\small\href{mailto:yz11751@stern.nyu.edu}{\textcolor{black}{\nolinkurl{yz11751@stern.nyu.edu}}}}
  \end{tabular}\qquad
  \begin{tabular}[t]{@{}c@{}}
    Will Ma\\
    {\small Graduate School of Business}\\
    {\small Columbia University}\\
    {\small\href{mailto:wm2428@gsb.columbia.edu}{\textcolor{black}{\nolinkurl{wm2428@gsb.columbia.edu}}}}
  \end{tabular}\\
  \noalign{\vskip 1em}
  \begin{tabular}[t]{@{}c@{}}
    Weihua Zhou\\
    {\small School of Management}\\
    {\small Zhejiang University}\\
    {\small\href{mailto:larryzhou@zju.edu.cn}{\textcolor{black}{\nolinkurl{larryzhou@zju.edu.cn}}}}
  \end{tabular}\qquad
  \begin{tabular}[t]{@{}c@{}}
    Zhengyuan Zhou\\
    {\small Stern School of Business}\\
    {\small New York University}\\
    {\small\href{mailto:zz26@stern.nyu.edu}{\textcolor{black}{\nolinkurl{zz26@stern.nyu.edu}}}}
  \end{tabular}
}
\date{}
\hypersetup{
  pdftitle={Adaptive Self-Consistency: From Black-Box Sampling to Distribution-Valued Feedback},
  pdfauthor={Jingkai Huang, Yunfan Zhang, Will Ma, Weihua Zhou, Zhengyuan Zhou}
}

\begin{document}
\maketitle

\begin{abstract}
Self-consistency samples many reasoning trajectories and aggregates their final answers, treating the LLM as a black box that returns one answer per trajectory. Yet the final answer of each trajectory is sampled from a softmax vector that is available from the model's log-probabilities. We refer to this as the grey-box setting in which each trajectory reveals this answer distribution rather than a single draw from it. We formulate efficient inference in this setting as sequential mode identification with distribution-valued observations: sample trajectories one at a time and stop as soon as the LLM's modal answer is identified at a prescribed confidence level. We characterize the asymptotic stopping rate of mode identification with distribution-valued observations exactly and show that it is never worse than the black-box rate. We then propose the ASC-D algorithm, a betting stopping rule that attains this asymptotic stopping rate.
On MMLU-Redux, ASC-D uses $46.4$--$95.6\%$ fewer trajectories than answer-only adaptive self-consistency baselines and achieves the highest fixed-budget correct-certification rate across three open-source models.

\end{abstract}

\section{Introduction}

Large language models (LLMs) have demonstrated strong performance on challenging mathematical and logical reasoning tasks. When solving a complex problem, an LLM often produces a sequence of intermediate reasoning steps before giving its final answer, a generation pattern commonly known as chain-of-thought reasoning \citep{wei2022chain,kojima2022large}. Because decoding is stochastic, repeatedly querying the same LLM can produce different reasoning trajectories and different final answers. Self-consistency (SC) exploits this diversity by sampling multiple reasoning trajectories and returning their most frequent final answer \citep{wang2022self}. Because each trajectory requires a separate LLM generation, the inference cost of SC grows with its sampling budget. Standard SC fixes this budget in advance, thus it may cause a computation waste on easy problems that require only a few trajectories. Adaptive self-consistency (ASC) addresses this inefficiency by sampling trajectories sequentially and stopping once the observed answers satisfy sufficient agreement \citep{aggarwal2023let}.

Both SC and ASC retain only the final-answer label from each reasoning trajectory, which we refer to as the \emph{black-box} observation model. Nevertheless, at the end of a trajectory, the probability vector from which the LLM decodes its final answer (token) is readily available and informative. Specifically, when Llama-3.2-3B-Instruct answers a four-choice ARC-Challenge question, we can read the next-token softmax over the four option tokens at the final-answer position, e.g., one recorded trajectory gives $(0.26,0.43,0.21,0.10)$ for options A--D, see Appendix~\ref{app:greybox-demo} for details on how we obtain the probabilities. This raises a natural question: can such distribution-valued feedback further reduce the number of trajectories required for reliable self-consistency?

We address this question by introducing adaptive self-consistency with \emph{distribution-valued feedback}. Figure~\ref{fig:black-box and grey-box} contrasts it with the black-box model. Both start from the same reasoning process: a prompt $x$ yields a sampled trajectory $r$, which induces a probability vector $\boldsymbol{\theta}$ over the $K$ candidate answers. The black-box model observes only one answer label drawn from $\boldsymbol{\theta}$, whereas our setting observes the vector itself (as well as the answer label). The two settings share the same answer frequencies and the same modal answer, and differ only in how much of $\boldsymbol{\theta}$ they reveal. Thus, we call our observation model \emph{grey-box}. In a \emph{white-box} setting, where the token-level probabilities for \textbf{all} possible reasoning trajectories are available, we could in principle identify the modal answer without repeated sampling by directly calculating the final answer probabilities. However, enumerating and evaluating all possible trajectories is generally infeasible. Grey-box feedback therefore offers a practical middle ground between black-box and white-box: it retains trajectory sampling while using each trajectory's probabilities for all candidate answers.

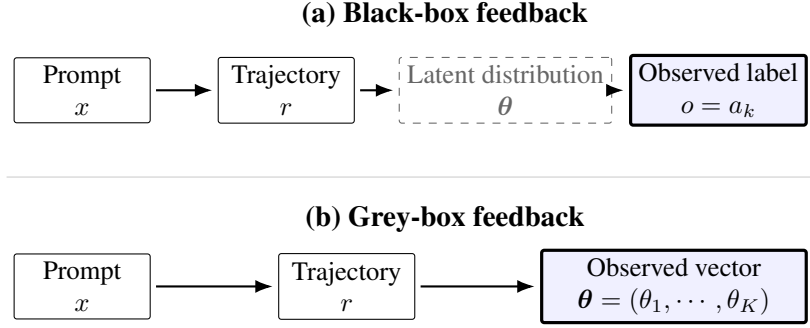
\begin{figure}[t]
\centering
\begin{tikzpicture}[
    >=Latex,
    font=\small,
    box/.style={
        draw,
        rounded corners=1pt,
        minimum width=1.8cm,
        minimum height=0.75cm,
        align=center,
        inner sep=3pt
    },
    observed/.style={
        box,
        very thick,
        fill=blue!6
    },
    unobserved/.style={
        box,
        dashed,
        draw=black!60,
        text=black!60
    },
    arrow/.style={
        ->,
        thick,
        shorten >=2pt,
        shorten <=2pt
    },
    paneltitle/.style={
        font=\bfseries
    }
]

\node[paneltitle] at (0,2.15)
{(a) Black-box feedback};

\node[box] (bx) at (-4.8,1.15)
{Prompt\\$x$};

\node[box] (br) at (-2.1,1.15)
{Trajectory\\$r$};

\node[
    unobserved,
    minimum width=2.5cm
] (btheta) at (0.8,1.15)
{Latent distribution\\$\boldsymbol{\theta}$};

\node[
    observed,
    minimum width=2.0cm,
    anchor=east
] (bo) at (4.8,1.15)
{Observed label\\$o=a_k$};

\draw[arrow] (bx.east) -- (br.west);
\draw[arrow] (br.east) -- (btheta.west);
\draw[arrow] (btheta.east) -- (bo.west);

\draw[gray!40] (-5.8,0) -- (5.8,0);

\node[paneltitle] at (0,-0.55)
{(b) Grey-box feedback};

\node[box] (gx) at (-4.8,-1.45)
{Prompt\\$x$};

\node[box] (gr) at (-1.3,-1.45)
{Trajectory\\$r$};

\node[
    observed,
    minimum width=3.5cm,
    anchor=east
] (gtheta) at (4.8,-1.45)
{Observed vector\\
$\boldsymbol{\theta}=(\theta_{1},\cdots,\theta_{K})$};

\draw[arrow] (gx.east) -- (gr.west);
\draw[arrow] (gr.east) -- (gtheta.west);

\end{tikzpicture}

\caption{Black-box and grey-box feedback from the same reasoning process.}
\label{fig:black-box and grey-box}
\end{figure}

Building on this richer observation model, we study how much grey-box feedback can reduce the number of trajectories required for reliable mode identification and whether this improvement can be attained by a practical adaptive stopping rule. We summarize our main contributions below.

\begin{enumerate}[noitemsep, topsep=0pt, left=1em]
    \item \textbf{The benefit of grey-box feedback.} For general distributions of answer-probability vectors, we derive a lower bound on the asymptotic sample complexity of any algorithm and show that grey-box feedback never requires more trajectories than black-box feedback. Specifically, for small frequency gap $\Delta$ between the modal answer and the runner-up, the sample complexity ranges from the black-box order $\Theta(\Delta^{-2})$, when every trajectory commits to a single answer, down to $\Theta(\Delta^{-1})$, when trajectories report the same beliefs.

    \item \textbf{An algorithm that fully realizes the benefit.} We develop an algorithm for adaptive self-consistency with distribution-valued feedback (ASC-D), the grey-box counterpart of ASC. Instead of counting answer labels, ASC-D bets on pairwise differences of the probability vectors and stops once the leading answer has accumulated enough evidence against every alternative. It controls the probability of returning a non-modal answer at level $\delta$ at every stopping time, and its expected number of trajectories attains the lower bound asymptotically, without any knowledge of the distribution of the probability vectors. The gain identified above is therefore attained in full by a practical procedure.

    \item \textbf{Empirical validation.}
    We validate our theoretical predictions through synthetic experiments and evaluations on MMLU-Redux. On MMLU-Redux, ASC-D uses $46.4$--$95.6\%$ fewer trajectories than answer-only baselines as the modal gap narrows. Under a fixed maximum budget of $N=128$, it achieves the highest correct-certification rate at every tested confidence level across Llama-3.2-1B, Llama-3.1-8B, and Qwen3-1.7B.
\end{enumerate}

\paragraph{Organization}  The rest of the paper is organized as follows. Section \ref{sec:related} reviews related literature. Section \ref{sec:setup} presents the problem setup, derives and compares asymptotic lower bounds on the expected stopping time under different observation models. Section \ref{sec:alg} introduces the ASC-D algorithm and its theoretical guarantees. Section \ref{sec:experiments} reports the experiment results on both synthetic and real-world datasets, and Section \ref{sec:conclusion} concludes with a discussion of limitations. Appendix \ref{App:main} provides the proof of main results in the paper. Appendix \ref{app:greybox-demo} gives a real-world example of grey-box feedback. Appendices \ref{app:synthetic-additional} and \ref{app:real-world-additional} provide omitted details and additional results for the synthetic and real-world experiments, respectively.

\section{Related Literature}\label{sec:related}

\paragraph{Efficient and Adaptive Self-Consistency}
Adaptive self-consistency reduces the cost of fixed-budget self-consistency by stopping when the sampled answers show sufficient agreement \citep{aggarwal2023let}. Subsequent methods stop based on answer stability \citep{li2024escape}, allocate sampling budgets using estimated question difficulty \citep{wang2025make}, or evaluate both answers and reasoning paths \citep{wan2025reasoning}. Other approaches use scalar confidence signals to weight or filter reasoning trajectories \citep{taubenfeld2025confidence,fu2026deep}. From a theoretical perspective, \citet{huang2025sample} and \citet{huang2026decision} compare the sample complexity of self-consistency and verifier-based best-of-$N$, showing a separation between their quadratic and linear dependence on the answer gap. Most closely related to our setting, \citet{huang2026optimal} studies Bayesian stopping using prior information about answer frequencies, but observes only the final-answer labels. In contrast, our method does not require such prior information and instead gains efficiency by observing the full probability vector associated with each trajectory.

\paragraph{Sequential Mode Identification and Anytime-Valid Inference.}
Our problem is also related to classical sequential hypothesis testing, which studies how to make a reliable decision using as few observations as possible \citep{wald1992sequential,siegmund2013sequential}. Sequential mode identification specializes this problem to identifying the most likely outcome of an unknown discrete distribution. \citet{shah2020sequential} study mode identification under different oracle-query models, while \citet{jain2022pac} develops an asymptotically optimal stopping rule using sampled answer labels and martingale confidence sequences. Recent work applies anytime-valid inference to LLM self-consistency: MMC certifies an absolute majority \citep{cordero2025certified}, whereas CITE certifies a candidate as the unique mode and also considers scalar confidence \citep{ota2026cite}. Our work extends this literature from categorical or scalar-weighted observations to probability-vector observations and characterizes the resulting optimal stopping rate.

\section{Problem Setup}\label{sec:setup}

\subsection{Distribution-Valued Feedback}\label{sec:feedback}

For a fixed prompt $x$ with a known set of candidate answers $\mathcal{A}=\{a_1,\cdots,a_K\}$, each LLM call first generates a reasoning trajectory $r_n$ and then produces a final answer. Formally, conditional on $r_n$ and $x$, the model assigns a probability to each candidate answer, yielding the trajectory-specific distribution
\[
\boldsymbol{\theta}_n
:=
\mathbb{P}(\,\cdot\mid r_n,x)
\in\Delta^{K-1},
\qquad
\theta_{n,k}
:=
\mathbb{P}(a_k\mid r_n,x).
\]
Standard black-box self-consistency returns only the realized answer $o_n\mid\boldsymbol{\theta}_n\sim\Cat(\boldsymbol{\theta}_n)$ sampled from the categorical distribution with parameter $\boldsymbol{\theta}_n$, whereas the grey-box setting observes the entire vector $\boldsymbol{\theta}_n$. Because different reasoning trajectories can induce different answer distributions, these vectors may vary across calls. We write $\boldsymbol{\theta}$ for a generic trajectory-specific distribution.

\paragraph{General Model} Since trajectories are sampled independently, the vectors $(\boldsymbol{\theta}_n)_{n\geq 1}$ are i.i.d. drawn from an unknown distribution $F$ on $\Delta^{K-1}$. Let $\mathcal{M}$ denote the set of all distributions on $\Delta^{K-1}$ and
$$
\boldsymbol{\pi} := \mathbb{E}_F[\boldsymbol{\theta}]=(p_1,\cdots,p_K)
$$
for the mean vector of $F$. We assume throughout that $F$ has a unique mode, and index the candidate answers so that $p_1>p_2\geq\cdots\geq p_K$, and $a_1$ is the modal answer. In this model, the black-box answers are i.i.d. with $\mathbb{P}(o_n=a_k)=\mathbb{E}_F[\theta_k]=p_k$. Thus, $\boldsymbol{\pi}$ is the answer-frequency vector of black-box self-consistency, and both observation models target the same modal answer $a_1$. %

\subsection{Sequential Mode Identification}\label{sec:mode}

The goal of self-consistency is to identify the modal answer $a_1$ of the LLM, which is the answer that a majority vote over infinitely many trajectories would return. A sequential procedure samples trajectories one at a time, observes $(\boldsymbol{\theta}_n)_{n\geq1}$ in the grey-box setting or $(o_n)_{n \geq 1}$ in the black-box setting, and consists of a stopping time $\tau$ with respect to the observations together with a recommendation $\hat a_\tau$. Stopping is the only decision: every trajectory costs one LLM call, and the procedure should stop as soon as the mode is identified at the prescribed error probability $\delta$.

\paragraph{$\delta$-PAC Procedures} A procedure is \emph{$\delta$-PAC over $\mathcal{M}$} if
$$
\mathbb{P}_F(\hat a_\tau\neq a_1)\leq\delta \qquad\text{for every } F\in\mathcal{M} \text{ with a unique mode},
$$
that is, its error guarantee holds whatever the trajectories do.

\paragraph{Characteristic Time} Following \citet{garivier2016optimal,jain2022pac}, we measure the efficiency of a $\delta$-PAC procedure by the asymptotic constant of its expected stopping time as the confidence grows, $\liminf_{\delta\to0}\mathbb{E}_F[\tau]/\log(1/\delta)$, and define the \emph{characteristic time} of $F$ as the best constant over all $\delta$-PAC procedures,
$$
T^\star(F):= \liminf_{\delta\to0}\dfrac{\mathbb{E}_F[\tau]}{\log(1/\delta)}.
$$
By definition, every $\delta$-PAC procedure needs at least $T^\star(F)\cdot\log(1/\delta)$ trajectories on $F$ asymptotically before stopping. Throughout we write
\begin{align}\label{eqn:notation}
\Delta := p_1-p_2, \qquad \bar p := \dfrac{p_1+p_2}{2}, \qquad X_j := \theta_1-\theta_j\in[-1,1] \quad (j\neq1),
\end{align}
where $\Delta$ is the gap between the top two answers and $X_j$ is the margin of the mode over its $j$-th challenger in a single trajectory.

\paragraph{The Characteristic Time under Black-Box} In the black-box setting, the observations are i.i.d. $\Cat(\boldsymbol{\pi})$ whatever the law $F$, so the procedure sees only $\boldsymbol{\pi}$. \citet{shah2020sequential,jain2022pac} show that the characteristic time of mode identification from categorical samples is
\begin{align}\label{eqn:shah}
T_{\rm bb}(F) := \dfrac{1}{p_1\cdot\log\frac{2p_1}{p_1+p_2}+p_2\cdot\log\frac{2p_2}{p_1+p_2}} = \dfrac{4\bar p}{\Delta^2}\cdot\big(1+O(\Delta^2)\big),
\end{align}
where the expansion holds as $\Delta\to 0$, see Appendix \ref{App:main} for details. We write $T_{\rm bb}(F)$ for this black-box characteristic time, which depends on $F$ only through $\boldsymbol{\pi}=\mathbb{E}_F[\boldsymbol{\theta}]$. The denominator is the Kullback--Leibler divergence from $\Cat(\boldsymbol{\pi})$ to the closest categorical law in which $a_2$ ties with $a_1$; the quadratic dependence on $\Delta$ is the price of resolving a small gap from the hard votes. Since grey-box feedback can always be reduced to black-box feedback by drawing $o \sim \Cat(\boldsymbol{\theta})$, \eqref{eqn:shah} is an upper bound on the grey-box characteristic time. Whether, and by how much, grey-box feedback improves on this bound is the question we turn to next.

\subsection{Asymptotic Lower Bounds}\label{sec:lower}

How much the grey-box feedback gain depends on how the probability vectors spread around their mean $\boldsymbol{\pi}$, and two extreme laws with this mean bracket the possibilities. At one end, the \emph{degenerate law} $\delta_{\boldsymbol{\pi}}$ puts all its mass at $\boldsymbol{\pi}$: every trajectory reports the same beliefs. At the other end, the \emph{vertex law} $F_0 := \sum_{k=1}^K p_k\cdot\delta_{\boldsymbol{e}_k}$ puts $\boldsymbol{\theta}$ at a vertex of the simplex with probability $p_k$: every trajectory commits to one answer, and $\boldsymbol{\theta}$ carries no more information than the sampled label $o$. The following theorem gives the grey-box characteristic time for a general $F$ and shows that these two laws are exactly the two extremes of the characteristic time.

\begin{theorem}[Asymptotic Lower Bound under Grey-box Feedback]\label{thm:main}
For every $\delta$-PAC procedure and every $F\in\mathcal{M}$ with a unique mode, under the grey-box feedback model,
\begin{align}\label{eqn:main}
\liminf_{\delta\to0}\dfrac{\mathbb{E}_F[\tau]}{\log(1/\delta)} \geq T_{\rm gb}(F) := \dfrac{1}{\min_{j=2,\cdots,K}\mathcal{G}_j(F)}, \quad \mathcal{G}_j(F) := \max_{\lambda\in[0,1]}\mathbb{E}_F\left[\log\left(1+\lambda\cdot(\theta_1-\theta_j)\right)\right].
\end{align}
where we use $T_{\rm gb}(F)$ to denote the characteristic time for distribution $F$ under the grey-box feedback model. Moreover,
\begin{align}\label{eqn:sandwich}
T_{\rm gb}(\delta_{\boldsymbol{\pi}}) = \dfrac{1}{\log(1+\Delta)} \leq T_{\rm gb}(F) \leq T_{\rm gb}(F_0) = T_{\rm bb}(F),
\end{align}
i.e., the two endpoints are the grey-box characteristic times of the degenerate law and the vertex law, the most and the least consistent laws with the same answer frequencies.
\end{theorem}

For the first part of the theorem: a standard change-of-measure argument \citep{kaufmann2016complexity} gives a lower bound through KL minimization over alternative distributions (in which $a_1$ is not the unique mode). Making a competitor $a_j$ at least as likely as $a_1$ is equivalent to making the mean of $X_j= \theta_1-\theta_j$ nonpositive. We show that it suffices to optimize over distributions of this margin. By utilizing the dual representation of \citet{honda2015non}, we then turn this infinite-dimensional optimization over probability distributions into the one-dimensional maximization over $\lambda\in[0,1]$ in (\ref{eqn:main}). Thus, $\mathcal{G}_j(F)$ is the minimum KL divergence needed for $a_j$ to be at least as likely as $a_1$, and the hardest competitor determines the characteristic time. The sandwich bound in (\ref{eqn:sandwich}) can then be derived simply through the Jensen's inequality.

\paragraph{The Value of the Grey-Box Observation Model} Under the vertex law $F_0$, every trajectory commits to a single answer and the grey-box observation carries no more information than the sampled answer; accordingly $T_{\rm gb}(F_0)=T_{\rm bb}(F)$, and the upper bound is attained exactly at this black-box endpoint. Any law that keeps mass away from the vertices stops strictly faster. The gain is bounded, however. Even when every trajectory reports $\boldsymbol{\pi}$ exactly, a $\delta$-PAC procedure cannot exclude the law that agrees with $\delta_{\boldsymbol{\pi}}$ on a fraction $1/(1+\Delta)$ of the trajectories and commits to $a_2$ on the rest, whose divergence from $\delta_{\boldsymbol{\pi}}$ is $\log(1+\Delta)$, and it says that without further assumptions the dependence on the gap improves from $\Theta(\Delta^{-2})$ to at best $\Theta(\Delta^{-1})$.

\paragraph{Connection with Best-of-$N$} Best-of-$N$ is another widely applied method for test-time compute, which samples $N$ trajectories and selects the answer with the highest verifier score \citep{cobbe2021training}. With an ideal verifier, it succeeds whenever a correct answer appears among the samples. Assuming that the modal answer is correct, \citet{huang2026decision} show that the worst-case sample complexity scales as $\Theta(\Delta^{-1})$ for best-of-$N$, compared with $\Theta(\Delta^{-2})$ for self-consistency.

In contrast, our Theorem~\ref{thm:main} shows that, at the degenerate law $\delta_{\boldsymbol{\pi}}$, grey-box self-consistency has characteristic time $1/\log(1+\Delta)=\Theta(\Delta^{-1})$. Thus, maximally consistent grey-box feedback yields the same linear inverse-gap dependence as the worst-case benchmark for ideal best-of-$N$, using the LLM’s own answer probabilities available under our distribution-valued observation model, without training or evaluating a separate verifier.

\paragraph{Connection with Confidence-Weighted ASC}
Between the black-box and grey-box observation models lies confidence-weighted self-consistency \citep{taubenfeld2025confidence,fu2026deep,ota2026cite}, in which each trajectory reveals its sampled answer together with the confidence it assigns to that answer. We formalize this feedback as observing the sampled answer together with its conditional probability given the reasoning trajectory:
$$
o\mid\boldsymbol{\theta}\sim\Cat(\boldsymbol{\theta}),\qquad w:=\theta_o\in(0,1].
$$
The pair $(o,w)$ can be generated from $\boldsymbol{\theta}$ and reduces to $o$ once $w$ is dropped. Thus, when all three feedback models are used to identify the same modal answer $a_1$, the reduction argument of Section~\ref{sec:mode} places the confidence-feedback characteristic time between the black-box and grey-box times. The following proposition characterizes this intermediate rate, where we write $Q_F$ for the law of $(o,w)$ for $\boldsymbol{\theta} \sim F$.

\begin{prop}[Asymptotic Lower Bound under Confidence Feedback]\label{prop:cw-characteristic}
For every $F\in\mathcal{M}$ with unique mode $a_1$, any procedure that observes $(o,w)$ and is $\delta$-PAC over $\mathcal{M}$ satisfies:
\begin{align}\label{eqn:cw-characteristic}
\liminf_{\delta\to0}\frac{\mathbb{E}_F[\tau]}{\log(1/\delta)} \geq T_{\rm cw}(F) := \left[\min_{j\neq 1}\inf_{\substack{F'\in\mathcal{M}:
\mathbb{E}_{F'}[\theta_j]\geq\mathbb{E}_{F'}[\theta_1]}}
\KL(Q_F\|Q_{F'})\right]^{-1}.
\end{align}
Moreover, $T_{\rm gb}(F)\leq T_{\rm cw}(F)\leq T_{\rm bb}(F)$.
\end{prop}
Proposition~\ref{prop:cw-characteristic} shows that confidence feedback can reduce the characteristic time relative to observing only the sampled answer, while revealing the full answer distribution can yield a further reduction.

\paragraph{The Gain under Dirichlet Heterogeneity} To further quantify these gains, we now consider a Dirichlet model. We hold the answer frequencies $\boldsymbol{\pi}$ fixed and vary a single concentration parameter $c$, which controls the agreement across trajectories. This lets us compare all three characteristic times on the same family of instances.
\begin{assumption}[Dirichlet Heterogeneity]\label{Ass:Dir}
Assume $F=\Dir(c\boldsymbol{\pi})$ for a concentration $c>0$, i.e., $\boldsymbol{\theta}\sim\Dir(cp_1,\cdots,cp_K)$.
\end{assumption}
Under Assumption \ref{Ass:Dir}, we still have $\mathbb{E}[\theta_k]=p_k$ and
\begin{align}\label{eqn:dirvar}
\Var(\theta_k)=\dfrac{p_k\cdot(1-p_k)}{c+1}, \qquad \Cov(\theta_k,\theta_j)=-\dfrac{p_k\cdot p_j}{c+1}\quad(k\neq j).
\end{align}
The parameter $c$ measures how much the trajectories agree with one another. As $c\to\infty$, $F$ converges to the degenerate law $\delta_{\boldsymbol{\pi}}$, and as $c\to0$ it converges to the vertex law $F_0$. Between the two endpoints, \eqref{eqn:dirvar} shows that observing $\theta_k$ instead of the one-hot coordinate of $o$, whose variance is $p_k(1-p_k)$, reduces the variance by exactly the factor $c+1$.

Assumption \ref{Ass:Dir} makes the gain explicit through the single parameter $c$, interpolates between the black-box endpoint $F_0$ and the fully consistent endpoint $\delta_{\boldsymbol{\pi}}$. Besides, although $T_{\rm gb}(\Dir(c\boldsymbol{\pi}))$ still has no closed form for general $\boldsymbol{\pi}$, in the hard regime $p_1\approx p_2$ it reduces to the black-box constant divided by a single factor, as shown in the following proposition.

\begin{prop}[Characteristic Time under the Dirichlet Assumption]\label{prop:hard}
Let $F=\Dir(c\boldsymbol{\pi})$. Fix $c>0$, $\bar p$ and $(p_3,\cdots,p_K)$, and let $\Delta\to0$. Then
\begin{align}\label{eqn:hard}
T_{\rm gb}(F) = \dfrac{T_{\rm bb}(F)}{c+1}\cdot\big(1+O(\Delta)\big),
\end{align}
and the maximizer in \eqref{eqn:main} is $\lambda^\star=\frac{(c+1)\Delta}{2\bar p}\cdot(1+O(\Delta))$ for $j=2$. For fixed $\Delta$ and $c \to \infty$, $T_{\rm gb}(F)\to1/\log(1+\Delta)$.
\end{prop}
The following table compares the three feedback models on the same Dirichlet instances. Confidence feedback gives an intermediate gain, with a substantial gap from grey-box feedback at moderate $c$. Besides, the approximation $T_{\rm bb}(F)/(c+1)$ closely matches $T_{\rm gb}(F)$ for small $c$. As $c$ increases, however, the approximation tends to zero, while $T_{\rm gb}(F)$ approaches the limit $1/\log(1+\Delta)$.

\begin{table}[htbp]
\caption{Characteristic times for $F=\Dir(c\boldsymbol{\pi})$ with $\boldsymbol{\pi}=(0.4,0.3,0.2,0.1)$. For all $c$, $T_{\rm bb}(F)=139.52$ and $1/\log(1+\Delta) = 10.49$. The values of $T_{\rm cw}(F)$ and $T_{\rm gb}(F)$ are computed numerically through simulation. The last row is the fixed-$c$, small-gap approximation in \eqref{eqn:hard}.}
\begin{center}
\setlength{\tabcolsep}{3pt}
\label{tab:rates}
\small
\begin{tabular}{@{}cccccccccccc@{}}
\toprule
$c$ & $0.01$ & $0.1$ & $1$ & $2$ & $3$ & $5$ & $10$ & $20$ & $50$ & $100$ & $1000$ \\
\midrule
$T_{\rm cw}(F)$ & $139.52$ & $139.50$ & $138.22$ & $125.66$ & $109.15$ & $84.99$ & $55.76$ & $33.70$ & $15.62$ & $11.85$ & $10.76$ \\
$T_{\rm gb}(F)$ & $138.15$ & $126.95$ & $70.33$ & $47.24$ & $35.70$ & $24.17$ & $15.01$ & $12.34$ & $11.16$ & $10.82$ & $10.52$ \\
$T_{\rm bb}(F)/(c+1)$ & $138.14$ & $126.84$ & $69.76$ & $46.51$ & $34.88$ & $23.25$ & $12.68$ & $6.64$ & $2.74$ & $1.38$ & $0.14$ \\
\bottomrule
\end{tabular}
\end{center}
\end{table}

\section{The ASC-D Algorithm}\label{sec:alg}

\subsection{From the Rate to the Algorithm}\label{sec:idea}

The rate in Theorem \ref{thm:main} tells us which statistic to compute: Consider a gambler who starts with wealth $1$ and, on every trajectory $i$, bets a fraction $\lambda$ of his wealth on ``$a_k$ beats $a_j$'', receiving $1+\lambda\cdot(\theta_{i,k}-\theta_{i,j})$ per unit bet. After $n$ trajectories his wealth is
\begin{align}\label{eqn:wealth}
W_n^{(k,j)}(\lambda) := \prod_{i=1}^n\big(1+\lambda\cdot(\theta_{i,k}-\theta_{i,j})\big), \qquad \lambda\in[0,1),
\end{align}
which is nonnegative because $\theta_{i,k}-\theta_{i,j}\geq-1$, and $\mathcal{G}_j(F)$ is the best expected log-wealth per trajectory of the bet on ``$a_1$ beats $a_j$''. The process has the two properties a stopping statistic needs.

\paragraph{Validity at Stopping Time} If $a_k$ does not beat $a_j$, i.e., $\mathbb{E}_F[\theta_k-\theta_j]\leq0$, then $\mathbb{E}_F[1+\lambda(\theta_{i,k}-\theta_{i,j})]\leq1$, so \eqref{eqn:wealth} is a nonnegative supermartingale with initial value $1$, and Ville's inequality gives
$$
\mathbb{P}_F\Big(\sup_{n\geq1}W_n^{(k,j)}(\lambda)\geq\dfrac{1}{\alpha}\Big)\leq\alpha.
$$
Wealth above $1/\alpha$ is therefore evidence at level $\alpha$ that $a_k$ beats $a_j$, and the guarantee holds uniformly over time, which is what a data-dependent stopping rule requires.

\paragraph{Growth at the Optimal Rate} If $a_1$ beats $a_j$, the law of large numbers gives $\frac{1}{n}\log W_n^{(1,j)}(\lambda)\to\mathbb{E}_F[\log(1+\lambda(\theta_1-\theta_j))]$, so at the optimal fraction the wealth grows at rate $\mathcal{G}_j(F)$ and crosses $1/\delta$ after $\log(1/\delta)/\mathcal{G}_j(F)$ trajectories, as illustrated by Theorem \ref{thm:main}.

However, we note that two quantities in this argument are unknown to the procedure: the optimal fraction $\lambda$ and the modal answer itself.

\paragraph{Unknown Optimal Fraction} The optimal fraction depends on $F$. As $F$ is itself unknown, we consider to average the wealth over a finite grid $\Lambda\subset[0,1)$,
\begin{align}\label{eqn:mixture}
\overline W_n^{(k,j)} := \dfrac{1}{|\Lambda|}\sum_{\lambda\in\Lambda}W_n^{(k,j)}(\lambda).
\end{align}
A convex combination of supermartingales is a supermartingale, so validity is preserved, and $\log\overline W_n^{(k,j)}\geq\max_{\lambda\in\Lambda}\log W_n^{(k,j)}(\lambda)-\log|\Lambda|$, so the averaged wealth grows at the best rate on the grid, $\max_{\lambda\in\Lambda}\mathbb{E}_F[\log(1+\lambda(\theta_1-\theta_j))]$, only at the price of a constant that does not affect the rate. The choice of the grid will be discuss in the following subsection.

\paragraph{Unknown Mode Answer} As the algorithm cannot know in advance which answer to bet on (i.e., which answer is the modal answer), so it calculates \eqref{eqn:mixture} for every ordered pair $(k,j)$. The current leader is the answer with the largest cumulative mass $S_{n,k}:=\sum_{i\leq n}\theta_{i,k}$, the grey-box counterpart of the plurality vote, and the procedure stops once the leader's averaged wealth against every challenger exceeds $(K-1)/\delta$. An error means that some $a_k\neq a_1$ is returned, which requires $\overline W_n^{(k,1)}\geq(K-1)/\delta$ for a pair in which $a_k$ does not beat $a_1$. As each of these $K-1$ events has probability at most $\delta/(K-1)$ by Ville's inequality, and a union bound to the $K-1$ components then gives the error probability $\delta$.

Algorithm~\ref{alg:ascd} provides the detailed procedure. If we use only the sampled answer from each trajectory instead of its full answer distribution, ASC-D reduces to the known-$K$ adaptation of CITE~\citep{ota2026cite}. Using full answer distributions allows ASC-D to also use the probabilities assigned to answers that were not sampled. Our later analysis will quantify this information gain and shows that ASC-D attains the optimal characteristic time $T_{\rm gb}(F)$ with a properly designed betting grid.

\begin{algorithm}[htbp]
\caption{ASC-D: adaptive self-consistency with distribution-valued feedback}\label{alg:ascd}
\begin{algorithmic}[1]
\Require prompt $x$, candidate answers $\{a_1,\cdots,a_K\}$, confidence $\delta$, grid $\Lambda$.
\State $\ell_{k,j}(\lambda)\gets0$ for all $k\neq j$ and $\lambda\in\Lambda$; \ $S_k\gets0$ for all $k$
\For{$n=1,2,\cdots$}
    \State sample a trajectory $r_n\sim\mathbb{P}(\cdot\mid x)$ and read off $\boldsymbol{\theta}_n=\mathbb{P}(\cdot\mid r_n,x)$ over the candidate set
    \State $S_k\gets S_k+\theta_{n,k}$ for all $k$ \Comment{cumulative mass of each answer}
    \State $\ell_{k,j}(\lambda)\gets\ell_{k,j}(\lambda)+\log\big(1+\lambda\cdot(\theta_{n,k}-\theta_{n,j})\big)$ for all $k\neq j$ and $\lambda\in\Lambda$
    \State $\hat k\gets\argmax_k S_k$ \Comment{current leader}
    \State $\log\overline W_j\gets\log\frac{1}{|\Lambda|}\sum_{\lambda\in\Lambda}\exp\big(\ell_{\hat k,j}(\lambda)\big)$ for all $j\neq\hat k$ \Comment{averaged against each challenger}
    \If{$\min_{j\neq\hat k}\log\overline W_j\geq\log\frac{K-1}{\delta}$}
        \State \Return $a_{\hat k}$
    \EndIf
\EndFor
\end{algorithmic}
\end{algorithm}

\subsection{Performance Guarantees}

\paragraph{Design of the Grid}  The grid $\Lambda$ should approximate each challenger's optimal betting fraction $\lambda_j^\star$ in \eqref{eqn:main}. Two insights guide its design (see Appendix \ref{App:main} for details). \emph{(i) The middle can be coarse};  \emph{(ii) Both ends should be dense}. We therefore utilize the geometric grid
\begin{align}\label{eqn:grid}
\Lambda_{r,m} := \big\{r^{-i}:\ i=1,\cdots,m\big\}\cup\big\{1-r^{-i}:\ i=1,\cdots,m\big\}, \qquad r>1,\ m\in\mathbb{N}.
\end{align}

\paragraph{Asymptotic Upper Bound of ASC-D} Define
\begin{align}\label{eqn:gridrate}
\mathcal{G}^{\Lambda}_j(F) := \max_{\lambda\in\Lambda}\mathbb{E}_F\big[\log\big(1+\lambda\cdot(\theta_1-\theta_j)\big)\big], \qquad T_{\rm gb}^{\Lambda}(F) := \dfrac{1}{\min_{j\neq1}\mathcal{G}^{\Lambda}_j(F)},
\end{align}
The following theorem provides the performance guarantees of Algorithm \ref{alg:ascd}.

\begin{theorem}[Performance Guarantees of ASC-D]\label{thm:ascd}
Let $(\tau^{\rm D},\hat a_{\tau^{\rm D}})$ be the stopping time and the output of Algorithm \ref{alg:ascd} with grid $\Lambda \subset[0,1)$, then for every $\delta\in(0,1)$, $\mathbb{P}_F(\tau^{\rm D}<\infty,\ \hat a_{\tau^{\rm D}}\neq a_1)\leq\delta$.
\begin{itemize}[noitemsep, topsep=0pt, left=1em]
\item[(i)] Let $\Lambda=\Lambda_{r,m}$ with $r^{-m}\leq\min_{j\neq1}\lambda^\star_j$. We have $\limsup_{\delta \to 0} \frac{\mathbb{E}_F[\tau^{\rm D}]}{\log(1/\delta)}\leq T_{\rm gb}^{\Lambda}(F)\leq r\cdot T_{\rm gb}(F)$.

\item[(ii)] Let $\Lambda=\Lambda_{r_\delta,m_\delta}$ with $r_\delta := 1+\big(\log(1/\delta)\big)^{-1/4}$, and $m_\delta := \big \lceil(\log(1/\delta))^{1/2}\big \rceil$. We have
\begin{align}\label{eqn:ascd-exact}
\limsup_{\delta\to0}\dfrac{\mathbb{E}_F[\tau^{\rm D}]}{\log(1/\delta)}\leq T_{\rm gb}(F).
\end{align}
\end{itemize}
\end{theorem}
Together with Theorem \ref{thm:main}, \eqref{eqn:ascd-exact} gives, for every $F$ with a unique mode, $\lim_{\delta\to0}\frac{\mathbb{E}_F[\tau^{\rm D}]}{\log(1/\delta)}=T_{\rm gb}(F)$. Thus, ASC-D can attain the optimal characteristic time of distribution-valued feedback exactly.

\section{Experiments}\label{sec:experiments}

We evaluate whether distribution-valued feedback improves the sample efficiency of mode identification in both synthetic and real-world settings. In both settings, all methods target the same population mode and are evaluated on the same trajectories. We test sample efficiency in two ways: the number of generations needed to certify the mode at a fixed confidence level (i.e., the stopping time), and the probability of correct certification of the mode within a fixed maximum trajectory budget $N$ (i.e., $\Prob(\hat a_\tau=a_1,\tau\le N)$). Section~\ref{sec:synthetic} studies the effect of the feedback type in synthetic datasets, and Section~\ref{sec:real-world} tests whether the same gains persist on the probability vectors generated from live LLM calls.

Code, stored probability vectors and trajectories, and experiment scripts are available at \url{https://anonymous.4open.science/r/ICLR2027-ASCD-7B09/}.

\subsection{Controlled Comparison of Feedback Types}\label{sec:synthetic}

We set $K=4$ and consider
$\boldsymbol{\pi}_{\Delta}=(0.35+\Delta/2,0.35-\Delta/2,0.15,0.15)$, where $0 < \Delta \leq 0.4$ is the probability gap between the top-two answers. At each time $t$, we draw $Z_t\sim\operatorname{Cat}(\boldsymbol{\pi}_{\Delta})$, form $\boldsymbol{\theta}_t=\frac{1}{2}\boldsymbol e_{Z_t} + \frac{1}{2}\boldsymbol{\pi}_{\Delta}$, and sample $Y_t\mid\boldsymbol{\theta}_t\sim\operatorname{Cat}(\boldsymbol{\theta}_t)$. This construction gives
$\mathbb{E}[\boldsymbol{\theta}_t]  = \boldsymbol{\pi}_{\Delta}$ and $Y_t \sim \operatorname{Cat}(\boldsymbol{\pi}_{\Delta})$. We compare the ASC-D with observation of $\boldsymbol{\theta}_t$, a confidence-weighted variant using $\theta_{t,Y_t} \cdot \boldsymbol e_{Y_t}$, and CITE-4~\citep{ota2026cite} using $\boldsymbol e_{Y_t}$ and its native certification rule. All procedures use the same 15-point grid, and we run 10,000 paired replays per condition. %

\begin{table}[htbp]
    \caption{Controlled comparison of feedback types.}
    \label{tab:synthetic-comparison}
    \begin{center}
    \small
    \renewcommand{\arraystretch}{1.05}

    \begin{minipage}[t]{0.49\linewidth}
        \centering
        \textbf{(a) Mean stopping time under $1-\delta=0.95$}\\[3pt]
            \begin{tabular}{@{}cccc@{}}
                \toprule
                $\Delta$ & ASC-D  & Conf.-wt. & CITE-4 \\
                \midrule
                0.10 & \textbf{188.42} & 444.57 & 907.72 \\
                0.20 & \textbf{46.20}  & 99.21  & 231.54 \\
                0.30 & \textbf{26.79} & 46.58  & 109.82 \\
                0.40 & \textbf{20.62}   & 32.12  & 71.72  \\
                \bottomrule
            \end{tabular}%
    \end{minipage}\hfill
    \begin{minipage}[t]{0.49\linewidth}
        \centering
        \textbf{(b) Correct-certification rate under $\Delta=0.1$}\\[3pt]
            \begin{tabular}{@{}cccc@{}}
                \toprule
                $1-\delta$ & ASC-D  & Conf.-wt. & CITE-4 \\
                \midrule
                0.90  & \textbf{98.44\%} & 71.37\% & 30.86\% \\
                0.95  & \textbf{97.49\%}  & 62.49\% & 24.22\% \\
                0.975 & \textbf{96.40\%} & 54.83\% & 18.38\% \\
                0.99  & \textbf{94.23\%} & 45.11\% & 12.73\% \\
                \bottomrule
            \end{tabular}%
    \end{minipage}
    \end{center}
\end{table}

Table~\ref{tab:synthetic-comparison}(a) shows that under the $95\%$ confidence level, ASC-D uses $35.8$--$57.6\%$ fewer than the confidence-weighted method, and $71.3$--$80.1\%$ fewer than CITE-4, with larger savings at smaller gaps. For the harder setting $\Delta=0.1$ with a maximum trajectory budget of $N=512$, Table~\ref{tab:synthetic-comparison}(b) shows that ASC-D correctly certifies the mode in $94.23$--$98.44\%$ of trials across the tested confidence levels, substantially more often than the reduced-feedback methods. Note that the correct-certification rate decreases as the confidence requirement $1-\delta$ becomes stricter, since the fixed budget constraint leaves less time to accumulate the required evidence. Together, these results establish ASC-D as the most sample-efficient among the evaluated methods. Appendix~\ref{app:synthetic-additional} further validates ASC-D's characteristic-time scaling and the concentration-dependent gain in Proposition~\ref{prop:hard}.

\subsection{Real-World Experiments}\label{sec:real-world}

\paragraph{Exp. 1: Required number of trajectory generations across modal gaps.}
For each of 120 MMLU-Redux questions, we use 256 Qwen3-1.7B trajectories and rescale their stored answer logits with $T_a\in\{4,8,16,32\}$, where larger $T_a$ flattens the vectors and narrows the modal gap. At each temperature, the rescaled vectors define the empirical population $\widehat F = 256^{-1} \sum_{i=1}^{256}\delta_{\boldsymbol\theta_i}$, with mean $\widehat{\boldsymbol\pi}=256^{-1}\sum_{i=1}^{256}\boldsymbol\theta_i$, mode index $a^\star = \arg\max_k\widehat\pi_k$, and modal gap $\Delta=\widehat\pi_{a^\star}-\max_{j\neq a^\star}\widehat\pi_j$. We then run 500 paired replays by sampling uniformly from $\widehat F$. ASC-D observes $\boldsymbol\theta_t$; its matched one-hot ablation observes $\boldsymbol e_{Y_t}$ with $Y_t \mid \boldsymbol\theta_t\sim\operatorname{Cat}(\boldsymbol\theta_t)$; CITE-4 and PPR-1v1~\citep{jain2022pac} observe the same $Y_t$. The number of trajectory generation (Num. Gen.) is the stopping time for each method under confidence level $95\%$, and answer accuracy (Ans. Acc.) is evaluated against the correct answer. We observe benefits of our ASC-D over other benchmarks among all the 120 questions. Specifically, we focus on the 90 questions that retain the same ordered top-two answers across temperatures and for which every method attains the 95\% success target within the maximum budget of 32,768 trajectories. Table \ref{tab:temperature-gap-scaling} reports the corresponding results on these 90 questions.

\begin{table}[htbp]
    \caption{Number of trajectory generations required for $95\%$ confidence and the answer accuracy on the 90 common questions. Here, $\Delta_{\rm geo}$ denotes the geometric mean of the modal gap.}
    \label{tab:temperature-gap-scaling}
    \begin{center}
    \small
    \setlength{\tabcolsep}{3pt}
    \renewcommand{\arraystretch}{1.08}
    \begin{tabular}{@{}cc*{4}{cc}@{}}
        \toprule
        & & \multicolumn{2}{c}{ASC-D}
          & \multicolumn{2}{c}{ASC-D (one-hot)}
          & \multicolumn{2}{c}{PPR-1v1}
          & \multicolumn{2}{c}{CITE-4} \\
        \cmidrule(lr){3-4}
        \cmidrule(lr){5-6}
        \cmidrule(lr){7-8}
        \cmidrule(lr){9-10}
        {\footnotesize $T_a$} & {\footnotesize $\Delta_{\mathrm{geo}}$}
        & {\footnotesize Num.\ Gen.} & {\footnotesize Ans.\ Acc.}
        & {\footnotesize Num.\ Gen.} & {\footnotesize Ans.\ Acc.}
        & {\footnotesize Num.\ Gen.} & {\footnotesize Ans.\ Acc.}
        & {\footnotesize Num.\ Gen.} & {\footnotesize Ans.\ Acc.} \\
        \midrule
        4  & 0.61
        & \textbf{24.67} & \textbf{72.51\%}
        & 41.82   & 71.83\%
        & 46.05   & 71.65\%
        & 58.19   & 71.70\% \\
        8  & 0.42
        & \textbf{27.82} & \textbf{72.40\%}
        & 90.73   & 70.97\%
        & 99.21   & 70.92\%
        & 113.92  & 70.94\% \\
        16 & 0.21
        & \textbf{42.17} & \textbf{72.36\%}
        & 346.55  & 70.78\%
        & 380.76  & 70.77\%
        & 403.87  & 70.78\% \\
        32 & 0.10
        & \textbf{76.87} & \textbf{72.16\%}
        & 1593.49 & 70.73\%
        & 1732.97 & 70.74\%
        & 1711.74 & 70.74\% \\
        \bottomrule
    \end{tabular}
    \end{center}
\end{table}

The results yield two main findings. First, as $T_a$ increases from $4$ to $32$, the geometric-mean gap decreases from $0.610$ to $0.098$, and ASC-D's advantage grows. It uses $41.0$--$95.2\%$ fewer trajectories than its matched one-hot ablation, $46.4$--$95.6\%$ fewer than PPR-1v1, and $57.6$--$95.5\%$ fewer than CITE-4. Because full-vector and one-hot ASC-D use the same stopping rule and differ only in the observed feedback, their comparison isolates the benefit of distribution-valued feedback. These savings come without a loss in answer accuracy.

Second, regressing each question's log required budget on $\log(1/\Delta)$ gives a median slope of $0.766$ for ASC-D, versus $1.969$, $1.966$, and $1.854$ for its one-hot ablation, PPR-1v1, and CITE-4, respectively. This finite-range comparison supports the predicted near-linear rather than quadratic dependence on the inverse gap, and that the ASC-D slope below one is not an asymptotic claim. Appendix~\ref{app:real-world-additional} reports bootstrap confidence intervals and the corresponding log--log plot.

\paragraph{Exp. 2: Correct certification within a fixed budget.}
At the natural answer temperature $T_a=1$, we fix the maximum budget at $N=128$ and evaluate the same 90 MMLU-Redux questions used in Exp.~1 across Qwen3-1.7B, Llama-3.2-1B, and Llama-3.1-8B. For each model and question, the first 128 probability vectors define $\widehat F$, $\widehat{\boldsymbol\pi}$, and $a^\star$ as in Exp. 1. We run 200 paired replays at $1-\delta\in\{0.95,0.975,0.99\}$. At each step, we sample $\boldsymbol\theta_t$ uniformly from $\widehat F$ and draw $Y_t\mid\boldsymbol\theta_t\sim\operatorname{Cat}(\boldsymbol\theta_t)$. ASC-D observes $\boldsymbol\theta_t$, while CITE-4 and PPR-1v1 observe the matched draw $Y_t$. A replay succeeds only if the method stops within the fixed maximum budget and returns $a^\star$. %

\begin{table}[htbp]
    \caption{Correct-certification rate within $N=128$ trajectories on the 90 common questions.}
    \label{tab:cross-model-certification}
    \begin{center}
    \setlength{\tabcolsep}{7pt}
    \renewcommand{\arraystretch}{1.08}
    \begin{tabular}{llccc}
        \toprule
        Model & Method
        & $1-\delta=0.95$
        & $1-\delta=0.975$
        & $1-\delta=0.99$ \\
        \midrule
        Llama-3.2-1B
            & ASC-D   & \textbf{56.57\%} & \textbf{53.38\%} & \textbf{49.83\%} \\
            & CITE-4  & 37.51\% & 35.21\% & 32.93\% \\
            & PPR-1v1 & 36.11\% & 34.14\% & 31.94\% \\
        \midrule
        Llama-3.1-8B
            & ASC-D   & \textbf{80.62\%} & \textbf{78.84\%} & \textbf{76.44\%} \\
            & CITE-4  & 71.69\% & 69.48\% & 66.94\% \\
            & PPR-1v1 & 70.34\% & 68.34\% & 66.01\% \\
        \midrule
        Qwen3-1.7B
            & ASC-D   & \textbf{90.75\%} & \textbf{89.83\%} & \textbf{88.77\%} \\
            & CITE-4  & 89.29\% & 88.39\% & 87.44\% \\
            & PPR-1v1 & 88.69\% & 87.97\% & 87.03\% \\
        \bottomrule
    \end{tabular}
    \end{center}
\end{table}

Table~\ref{tab:cross-model-certification} shows that ASC-D achieves the highest correct-certification rate for every model and confidence level. Its advantage is largest on Llama-3.2-1B, where it exceeds CITE-4 and PPR-1v1 by up to $19.06$ percentage points and $20.46$ percentage points, respectively, highlighting the value of distribution-valued feedback. %

\FloatBarrier
\section{Conclusion and Limitations}\label{sec:conclusion}

In this paper, we show that observing the LLM's probabilities for all candidate answers on each trajectory can reduce the number of trajectories needed to identify the LLM's most consistent answer. Our analysis characterizes the optimal asymptotic rate and proposes ASC-D, an algorithm that attains this rate with a suitable betting grid.

The limitations of this work are as follows: (i) ASC-D identifies the LLM's most consistent answer, which need not be the correct answer, and therefore inherits self-consistency's failure mode when the model is confidently wrong; (ii) the method requires log-probabilities over a finite candidate set, covering multiple-choice and short-answer tasks with specified candidates but not directly open-ended generation with an unbounded answer space; and (iii) our experiments are limited to small open-source models and multiple-choice benchmarks, with evaluation on larger models and open-ended reasoning tasks left for future work.

\FloatBarrier

\newpage
\IfFileExists{Arxiv-Template/ASC-D_arxiv.bib}{%
  \bibliography{Arxiv-Template/ASC-D_arxiv}%
}{%
  \bibliography{ASC-D_arxiv}%
}
\bibliographystyle{plainnat}

\newpage

\appendix

\section{Proof of Main Results}\label{App:main}

\paragraph{Characteristic Time for the Black-box Observation Model} The reciprocal of the black-box constant is the KL divergence from $\Cat(\pi)$ to the merge alternative,
$$
\dfrac{1}{T_{\rm bb}(F)} = p_1\log\dfrac{p_1}{\bar p}+p_2\log\dfrac{p_2}{\bar p} = g\Big(\bar p+\dfrac{\Delta}{2}\Big)+g\Big(\bar p-\dfrac{\Delta}{2}\Big), \qquad g(x):=x\log\dfrac{x}{\bar p}.
$$
Expanding $g$ around $\bar p$, we have $g(\bar p)=0$, $g'(\bar p)=1$ and $g''(\bar p)=1/\bar p$. The first-order terms of the two summands are $+\frac{\Delta}{2}$ and $-\frac{\Delta}{2}$ and cancel, so only the second-order terms remain:
$$
\dfrac{1}{T_{\rm bb}(F)} = 2\cdot\dfrac{1}{2}\,g''(\bar p)\Big(\dfrac{\Delta}{2}\Big)^2+O(\Delta^4) = \dfrac{\Delta^2}{4\bar p}+O(\Delta^4),
\qquad\text{hence}\qquad
T_{\rm bb}(F)=\dfrac{4\bar p}{\Delta^2}\cdot\big(1+O(\Delta^2)\big).
$$
The cancellation is the whole reason for the quadratic rate: the merge moves $p_1$ and $p_2$ by the same amount in opposite directions, so the divergence has no linear part in $\Delta$.

\begin{proof}[Proof of Theorem \ref{thm:main}]

For distribution $F$ with unique mode $a_1$, we define the alternative set $\Alt(F)$ as
$$
\Alt(F):=\{F'\in\mathcal{M}:\argmax_k\mathbb{E}_{F'}[\theta_k]\neq1\}=\bigcup_{j=2}^K\{F':\mathbb{E}_{F'}[X_j]\leq0\},
$$
i.e., $\Alt(F)$ stands for the set of distributions in $\mathcal{M}$ for which $a_1$ is not the unique mode. We let $\Alt_j(F) := \{F':\mathbb{E}_{F'}[X_j]\leq0\}$ for all $j \neq 1$. Thus $\Alt(F) = \cup_{j=2}^K \Alt_j(F)$.

First, take an alternative $F' \in \Alt(F)$, and let $E$ be the event that the procedure recommends $a_1$ as the unique mode. We have the procedure is $\delta$-PAC at both $F$ and the alternative $F'$, i.e.,
$$
\mathbb{P}_F(E) \geq 1-\delta, \qquad \mathbb{P}_{F'}(E) \leq \delta.
$$
Then, by applying Lemma 1 of \citet{kaufmann2016complexity}, we can conclude the lower bound of:
$$
\mathbb{E}_F[\tau] \cdot \KL(F\|F') \geq \KL (\operatorname{Bern}(\mathbb{P}_F(E) \| \mathbb{P}_{F'}(E))) \geq (1-2\delta)\cdot \log \dfrac{1-\delta}{\delta}.
$$
Therefore, taking the supremum over the alternative set, we obtain
\begin{align}\label{m1}
\liminf_{\delta\to0}\dfrac{\mathbb{E}_F[\tau]}{\log(1/\delta)} \geq \dfrac{1}{\inf_{F' \in \Alt(F)} \KL(F \| F')}.
\end{align}

Next, it suffices to calculate the denominator in (\ref{m1}). We have
\begin{align}\label{m2}
\begin{aligned}
\inf_{F' \in \Alt(F)} \KL(F \| F') = \min_{j \neq 1} \inf_{F': \mathbb{E}_{F'}[X_j] \leq 0} \KL(F \| F')
\end{aligned}
\end{align}
according to the definition of $\Alt(F)$. For a distribution $G$ on $[-1,1]$, we define
\begin{align}\label{eqn:Kinf}
\Kinf(G) := \inf\left\{\KL(G\|G'): {\rm supp}\,G'\subseteq[-1,1],\ \mathbb{E}_{G'}[X]\leq0\right\},
\end{align}
which represents the minimal KL divergence from the distribution $G$ to the set of candidate distributions $G'$ supported on $[-1,1]$ with a mean less than or equal to 0. According to Proposition~1 of \citet{honda2015non}, we have for any distribution $G$ on $[-1,1]$ with positive mean,
$$
\Kinf(G) = \max_{\lambda \in [0,1]} \mathbb{E}_G [\log (1+\lambda X)].
$$

Fix $j \neq 1$, and let $G_j$ be the law of $X_j$ under $F$. For every $F'\in \Alt_j(F)$, the law $G'_j$ of $X_j$ under $F'$ has nonpositive mean. Data processing inequality therefore implies
$$
\KL(F\|F') \geq \KL(G_j\|G'_j) \geq \Kinf(G_j).
$$
To obtain the reverse inequality, we construct an alternative $F'$ that differs from $F$ only in the distribution of $X_j$. Sampling from $F$ can be viewed as two steps: first draw $X_j\sim G_j$, and then draw $\theta$ from its conditional distribution under $F$ given $X_j$. Now we take any candidate law $G'$ in (\ref{eqn:Kinf}). Replace the first step by drawing $X_j\sim G'$, while leaving the second step unchanged. Let $F'$ denote the resulting distribution of $\theta$.

Under $F'$, the margin $X_j$ has exactly the law $G'$, whose mean is nonpositive. Thus $F'\in\Alt_j(F)$. Moreover, the conditional distribution used in the second step is the same under $F$ and $F'$, so it contributes no additional KL divergence. The chain rule for relative entropy therefore gives
$$
\KL(F\|F')=\KL(G_j\|G').
$$
Hence every candidate $G'$ in the one-dimensional problem gives a candidate $F'$ in the original problem with the same KL divergence. Taking the infimum over $G'$ proves the reverse inequality:
$$
\inf_{F'\in\Alt_j(F)}\KL(F\|F')
\leq\inf_{G':\mathbb{E}_{G'}[X]\leq0}\KL(G_j\|G')
=\Kinf(G_j).
$$
Combining these two directions of the inequalities yields the equation of
$$
\inf_{F'\in \Alt_j(F)}\KL(F\|F')
=\Kinf(G_j) =\max_{\lambda\in[0,1]}\mathbb{E}_F[\log(1+\lambda X_j)] =\mathcal{G}_j(F).
$$
Taking the result back to (\ref{m1}) and (\ref{m2}), we obtain the main statement of the theorem.

We then focus on the inequalities of
\begin{align}\label{sandwich}
\dfrac{1}{T_{\rm bb}(F)} \leq \min_{j \neq 1} \mathcal{G}_j(F)  \leq \log(1+\Delta).
\end{align}
For the RHS of (\ref{sandwich}): for any $F$ and $\lambda \in [0,1]$, since $\log x$ is concave in $x$, by applying the Jensen's inequality,
$$
\mathbb{E}_F [\log (1+\lambda X_j)] \leq \log (1+ \lambda \mathbb{E}_F[X_j]) = \log (1+\lambda\cdot (p_1-p_j)) \leq \log (1+p_1-p_j).
$$
Let $j=2$ and we obtain $\min_{j \neq 1} \mathcal{G}_j(F)  \leq \log(1+\Delta)$. For $F=\delta_{\boldsymbol{\pi}}$, every margin is the positive constant $p_1-p_j$, therefore $\mathcal{G}_j(\delta_{\boldsymbol{\pi}}) = \log(1+p_1-p_j)$ and thus $F=\delta_{\boldsymbol{\pi}}$ can achieve the upper bound.

For the LHS of (\ref{sandwich}): concavity of the logarithm and $\sum_k\theta_k=1$ give, for $0\leq\lambda<1$,
\begin{align*}
\log(1+\lambda(\theta_1-\theta_j))
&=\log\left(\theta_1(1+\lambda)+\theta_j(1-\lambda)
+\sum_{k\notin\{1,j\}}\theta_k\right)\\
&\geq\theta_1\log(1+\lambda)+\theta_j\log(1-\lambda).
\end{align*}
Taking expectations and optimizing over $\lambda$ yields
\begin{align}\label{eqn:categorical-pairwise-rate}
\mathcal{G}_j(F)
&\geq\sup_{0\leq\lambda<1}\{p_1\log(1+\lambda)+p_j\log(1-\lambda)\}\notag\\
&=D_j:=p_1\log\frac{2p_1}{p_1+p_j}
+p_j\log\frac{2p_j}{p_1+p_j},
\end{align}
with the maximizer $\lambda=(p_1-p_j)/(p_1+p_j)$. Because $p_j\leq p_2$ and $\log(1-\lambda)\leq0$, the supremum in (\ref{eqn:categorical-pairwise-rate}) is at least its value with $p_j$ replaced by $p_2$. Thus $D_j\geq D_2$ for every $j\neq 1$, and
$$
\min_{j\neq1}\mathcal{G}_j(F)\geq D_2=\frac{1}{T_{\rm bb}(F)}.
$$
Finally, under the vertex law $F_0=\sum_k p_k\delta_{\boldsymbol{e}_k}$, the margin $X_j$ equals $1$, $-1$, and $0$ with probabilities $p_1$, $p_j$, and $1-p_1-p_j$, respectively. Hence $\mathcal{G}_j(F_0)=D_j$, and thus the vertex law $F_0$ can achieve the lower bound.

\end{proof}

\begin{proof}[Proof of Proposition \ref{prop:cw-characteristic}]

Let $F\in\mathcal{M}$ have unique mode $a_1$, and recall that $Q_F$ is the distribution of the observed pair $(o,w)$. We use the same alternative sets as in the proof of Theorem \ref{thm:main}:
$$
\Alt_j(F):=\{F'\in\mathcal{M}:\mathbb{E}_{F'}[X_j]\leq0 \}, \qquad \Alt(F)=\bigcup_{j\neq1}\Alt_j(F),
$$
where $X_j=\theta_1-\theta_j$. Thus $\Alt(F)$ contains the distributions for which $a_1$ is not the unique mode.

First take an alternative $F'\in\Alt(F)$ with a unique mode different from $a_1$, and let $E$ be the event that the procedure recommends $a_1$. Since the procedure is $\delta$-PAC under both $F$ and $F'$,
$$
\mathbb{P}_F(E)\geq1-\delta, \qquad \mathbb{P}_{F'}(E)\leq \delta.
$$
The same change-of-measure argument used in the proof of Lemma 1 \citet{kaufmann2016complexity}, now applied to observations with laws $Q_F$ and $Q_{F'}$, gives,
$$
\mathbb{E}_F[\tau] \cdot \KL(Q_F\|Q_{F'}) \geq \KL (\operatorname{Bern}(\mathbb{P}_F(E) \| \mathbb{P}_{F'}(E))) \geq (1-2\delta)\cdot \log \dfrac{1-\delta}{\delta}.
$$
Similarly, taking the supremum over the alternative set, and letting $\delta \rightarrow 0$, we obtain
\begin{align*}
\liminf_{\delta\to0}\dfrac{\mathbb{E}_F[\tau]}{\log(1/\delta)} &\geq\frac{1}{\inf_{F'\in\Alt(F)}\KL(Q_F\|Q_{F'})}\\
&=\left[\min_{j\neq1}\inf_{F'\in\Alt_j(F)}\KL(Q_F\|Q_{F'})\right]^{-1} =T_{\rm cw}(F).
\end{align*}

Finally, for the sandwich bound $T_{\rm gb}(F)\leq T_{\rm cw}(F)\leq T_{\rm bb}(F)$: the pair $(o,w)$ can be generated from $\boldsymbol{\theta}$, and the sampled answer $o$ is obtained by dropping $w$. The data processing inequality therefore gives, for every $F'\in\Alt(F)$,
$$
\KL(F\|F')\geq\KL(Q_F\|Q_{F'}) \geq\KL\left(\Cat(\boldsymbol{\pi}) \| \Cat(\mathbb{E}_{F'}[\boldsymbol{\theta}])\right).
$$
Taking infima over the same alternative set, the first term gives $1/T_{\rm gb}(F)$ by Theorem~\ref{thm:main}. The last term gives $1/T_{\rm bb}(F)$ in \eqref{eqn:shah}. Thus we obtain the desired result. Besides, we also note that this asymptotic lower bound is achievable through some properly designed algorithm (with similar idea to that of our ASC-D). We omit the details as this is not the main focus of the paper.
\end{proof}

\begin{proof}[Proof of Proposition \ref{prop:hard}]
Throughout, $F=\Dir(c\pi)$ and $\phi_j(\lambda):=\mathbb{E}_F[\log(1+\lambda X_j)]$, so that $\mathcal{G}_j(F)=\max_{\lambda\in[0,1]}\phi_j(\lambda)$.

\emph{Second moment of $X_2$.} By \eqref{eqn:dirvar},
$$
\Var(X_2)=\Var(\theta_1)+\Var(\theta_2)-2\Cov(\theta_1,\theta_2)=\dfrac{p_1(1-p_1)+p_2(1-p_2)+2p_1p_2}{c+1}=\dfrac{p_1+p_2-\Delta^2}{c+1},
$$
and since $\mathbb{E}[X_2]=\Delta$,
\begin{align}\label{eqn:X2sq}
\mathbb{E}[X_2^2]=\dfrac{2\bar p-\Delta^2}{c+1}+\Delta^2=\dfrac{2\bar p}{c+1}\cdot\big(1+O(\Delta^2)\big).
\end{align}

\emph{Expansion of $\mathcal{G}_2$.} The function $\phi_2$ is concave on $[0,1]$, with $\phi_2'(\lambda)=\mathbb{E}[X_2/(1+\lambda X_2)]$ for $0 \leq \lambda < 1$. Fix $\lambda_0 \in(0,1)$. Since $1+\lambda_0X_2\leq 2$,
$$
\phi_2'(\lambda_0)=\mathbb{E}\Big[\dfrac{X_2}{1+\lambda_0X_2}\Big] = \mathbb{E}\Big[X_2-\dfrac{\lambda_0 X_2^2}{1+\lambda_0X_2}\Big]\leq\Delta-\dfrac{\lambda_0}{2}\,\mathbb{E}[X_2^2],
$$
which is negative for small $\Delta$ by \eqref{eqn:X2sq}. Since $\phi_2'(0)=\Delta > 0$, the maximizer $\lambda^\star$ lies in $(0,\lambda_0)$, and we have:
$$
\Delta=\lambda^\star\mathbb{E}\left[\dfrac{X_2^2}{1+\lambda^\star X_2}\right] \geq\dfrac{\lambda^\star}{2}\mathbb{E}[X_2^2],
$$
so $\lambda^\star\leq2\Delta/\mathbb{E}[X_2^2]=O(\Delta)$.

On this interval, $\log(1+u)=u-\frac{u^2}{2}+R(u)$ with $|R(u)|\leq C_0|u|^3$ for some constant $C_0 >0$ for any $|u|\leq\lambda_0$, and $|X_2|^3\leq X_2^2$ since $|X_2|\leq 1$, so with $u=\lambda X_2$,
\begin{align}\label{eqn:phi-expansion}
\phi_2(\lambda)=\lambda\Delta-\dfrac{\lambda^2}{2}\,\mathbb{E}[X_2^2]+r(\lambda), \qquad |r(\lambda)|\leq C_0\lambda^3\,\mathbb{E}[X_2^2].
\end{align}
The quadratic part is maximized at $\hat\lambda:=\Delta/\mathbb{E}[X_2^2]=\frac{(c+1)\Delta}{2\bar p}\cdot(1+O(\Delta^2))$, which lies in $[0,\lambda_0]$ for small $\Delta$, with value $\frac{\Delta^2}{2\mathbb{E}[X_2^2]}=\frac{(c+1)\Delta^2}{4\bar p}\cdot(1+O(\Delta^2))$. Evaluating \eqref{eqn:phi-expansion} at $\hat\lambda$ gives $\mathcal{G}_2(F)\geq\frac{(c+1)\Delta^2}{4\bar p}(1+O(\Delta^2))-C_0\hat\lambda^3\mathbb{E}[X_2^2]$, and the remainder is $O(\Delta^3)$ since $\hat\lambda=O(\Delta)$. Conversely, the quadratic part at $\lambda^\star$ is at most its value at $\hat\lambda$, and $r(\lambda^\star)=O(\Delta^3)$ because $\lambda^\star=O(\Delta)$. Thus $\mathcal{G}_2(F)\leq\frac{\Delta^2}{2\mathbb{E}[X_2^2]}+O(\Delta^3)$. Therefore
\begin{align}\label{eqn:G2-hard}
\mathcal{G}_2(F)=\dfrac{(c+1)\Delta^2}{4\bar p}\cdot\big(1+O(\Delta)\big).
\end{align}
For the maximizer, $\phi_2'(\lambda)=\mathbb{E}[X_2]-\lambda\mathbb{E}[X_2^2]+O(\lambda^2\mathbb{E}[X_2^2])$ for $\lambda\leq\lambda_0$, so the stationary condition $\phi_2'(\lambda^\star)=0$ gives $\lambda^\star=\frac{\Delta}{\mathbb{E}[X_2^2]}\cdot(1+O(\lambda^\star))=\frac{(c+1)\Delta}{2\bar p}\cdot(1+O(\Delta))$.

Along this limit, $p_1=\bar p+\Delta/2$ and $p_2=\bar p-\Delta/2$. For $K>2$, the fixed probabilities satisfy $p_j<\bar p$ for every $j\geq3$, since $p_j\leq p_2$. The pairwise lower bound $\mathcal{G}_j(F)\geq D_j$ in \eqref{eqn:categorical-pairwise-rate} then gives $\liminf_{\Delta\to0}\mathcal{G}_j(F)>0$ for each $j\geq3$, whereas \eqref{eqn:G2-hard} gives $\mathcal{G}_2(F) \to 0$. Thus $\min_{j\neq1}\mathcal{G}_j(F)=\mathcal{G}_2(F)$ for sufficiently small $\Delta$. For $K=2$ this result is immediate. Combining \eqref{eqn:G2-hard} with the black-box expansion in \eqref{eqn:shah}, we thus obtain now
\begin{align*}
T_{\rm gb}(F)
&=\dfrac{1}{\mathcal{G}_2(F)}
=\dfrac{4\bar p}{(c+1) \Delta^2} \big(1+O(\Delta)\big)\\
&=\dfrac{T_{\rm bb}(F)}{c+1}\big(1+O(\Delta)\big).
\end{align*}

\end{proof}

\paragraph{Insights on the Grid Design}

The averaged wealth \eqref{eqn:mixture} grows at the best rate on the grid. For each challenger $j\neq1$, let $\lambda_j^\star\in(0,1]$ be a maximizer in \eqref{eqn:main}, and write
$$
\phi_j(\lambda):=\mathbb{E}_F[\log(1+\lambda(\theta_1-\theta_j))], \qquad \mathcal{G}_j(F)=\phi_j(\lambda_j^\star).
$$
The optimal fraction depends on the unknown law $F$, so the grid must cover a range of possible values. The two insights are:
\begin{itemize}[noitemsep, topsep=0pt, left=1em]
    \item[(i)] \textbf{The middle can be coarse}. The function $\phi_j$ is concave and satisfies $\phi_j(0)=0$. Hence
$$
\phi_j(\lambda)\geq\frac{\lambda}{\lambda_j^\star}\mathcal{G}_j(F), \qquad 0\leq\lambda\leq\lambda_j^\star.
$$
Thus a point in $[\lambda_j^\star/r,\lambda_j^\star]$, for $r>1$, retains at least a fraction $1/r$ of the optimal growth rate. The width of this interval is $\lambda_j^\star(1-1/r)$, so relatively coarse spacing suffices when the optimum lies in the middle of $[0,1]$.
    \item[(ii)] \textbf{Both ends should be dense}. When $\lambda_j^\star$ is small, the interval $[\lambda_j^\star/r,\lambda_j^\star]$ is narrow. Under Assumption~\ref{Ass:Dir}, for example, $\lambda_2^\star\approx(c+1)\Delta/(2\bar p)$ tends to zero with the gap for fixed $c$ and $\bar p$. A uniform grid with spacing $h$ cannot provide a positive point below the optimum once $\lambda_j^\star<h$. Geometrically spaced points $r^{-i}$ instead cover progressively smaller fractions. At the other endpoint, highly consistent trajectories can favor fractions close to $1$. We therefore also include points $1-r^{-i}$, which approach $1$ for large $i$.
\end{itemize}

\begin{proof}[Proof of Theorem \ref{thm:ascd}]

We write $\mathcal{F}_n:=\sigma(\boldsymbol{\theta}_i:i\leq n)$ and $M^{(k,j)}_i:=\theta_{i,k}-\theta_{i,j}\in[-1,1]$, so that $W_n^{(k,j)}(\lambda)=\prod_{i\leq n}(1+\lambda M^{(k,j)}_i)$ and $\mathbb{E}_F[M^{(k,j)}_i]=p_k-p_j$. Throughout, $\phi_j(\lambda):=\mathbb{E}_F[\log(1+\lambda M^{(1,j)}_1)]$, so that $\mathcal{G}^\Lambda_j(F)=\max_{\lambda\in\Lambda}\phi_j(\lambda)$ in \eqref{eqn:gridrate}, and for a finite grid $\Lambda\subset[0,1)$ we further define
$$
g_\Lambda:=\min_{j\neq1}\mathcal{G}^{\Lambda}_j(F)=\dfrac{1}{T_{\rm gb}^{\Lambda}(F)}, \qquad B_\Lambda:=\max_{\lambda\in\Lambda}\log\dfrac{1}{1-\lambda}.
$$
Fix $k\neq1$ and $\lambda\in\Lambda$. Since $p_k-p_1<0$,
$$
\mathbb{E}_F\big[W_n^{(k,1)}(\lambda)\mid\mathcal{F}_{n-1}\big]=W_{n-1}^{(k,1)}(\lambda)\cdot\big(1+\lambda\cdot(p_k-p_1)\big)\leq W_{n-1}^{(k,1)}(\lambda),
$$
so $(W_n^{(k,1)}(\lambda))_{n \geq 0}$ is a nonnegative supermartingale with $W_0^{(k,1)}(\lambda)=1$, and so is its average $\overline W_n^{(k,1)}$ over $\lambda\in\Lambda$ defined in \eqref{eqn:mixture}. Ville's inequality for the supermartingale gives
$$
\mathbb{P}_F\Big(\sup_{n\geq1}\overline W_n^{(k,1)}\geq\dfrac{K-1}{\delta}\Big)\leq\dfrac{\delta}{K-1}.
$$
On the event $\{\tau^{\mathrm D}<\infty,\hat a_{\tau^{\mathrm D}}=a_k\}$ the stopping condition of Algorithm \ref{alg:ascd} holds at time $\tau^{\mathrm D}$ with leader $\hat k_{\tau^{\mathrm D}}=k$, and in particular $\overline W_{\tau^{\mathrm D}}^{(k,1)}\geq(K-1)/\delta$. Hence
$$
\mathbb{P}_F(\tau^{\mathrm D}<\infty,\ \hat a_{\tau^{\mathrm D}}\neq a_1)\leq\sum_{k=2}^K\mathbb{P}_F\Big(\sup_{n\geq1}\overline W_n^{(k,1)}\geq\dfrac{K-1}{\delta}\Big)\leq\delta.
$$

\emph{Part (i).} Let $\Lambda=\Lambda_{r,m}$ with $r^{-m}\leq\min_{j\neq1}\lambda^\star_j$. Fix $j \neq 1$, $\phi_j$ is concave on $[0,1]$ with $\phi_j(0)=0$ and $\phi_j(\lambda_j^\star)=\mathcal{G}_j(F)>0$, so $\lambda^\star_j>0$. The grid contains a point $\lambda\in[\lambda_j^\star/r,\lambda_j^\star]$: if $\lambda_j^\star<r^{-1}$, then $r^{-i}\leq\lambda_j^\star<r^{-i+1}$ for some $2\leq i\leq m$ because $r^{-m}\leq\lambda_j^\star$, and $\lambda=r^{-i}$ falls within the bandwidth. If $\lambda_j^\star\geq r^{-1}$, then $\lambda=r^{-1}$ falls within the bandwidth because $\lambda_j^\star \leq 1$. For such $\lambda$, concavity and $\phi_j(0)=0$ give
$$
\phi_j(\lambda)=\phi_j\Big(\dfrac{\lambda}{\lambda_j^\star}\cdot\lambda_j^\star+\Big(1-\dfrac{\lambda}{\lambda_j^\star}\Big)\cdot0\Big)\geq\dfrac{\lambda}{\lambda_j^\star}\cdot\mathcal{G}_j(F)\geq\dfrac{\mathcal{G}_j(F)}{r}.
$$
Hence $\mathcal{G}_j^\Lambda(F)\geq\mathcal{G}_j(F)/r$ for every $j\neq 1$, i.e.,
\begin{align}\label{eqn:gridloss}
g_{\Lambda_{r,m}}\geq\dfrac{1}{r\cdot T_{\rm gb}(F)}>0 \qquad\text{and}\qquad T_{\rm gb}(F)\leq T_{\Lambda_{r,m}}(F)\leq r\cdot T_{\rm gb}(F).
\end{align}

Next, we show that every finite grid $\Lambda$ with $g_\Lambda>0$ satisfies $\limsup_{\delta\to0}\mathbb{E}_F[\tau^{\mathrm D}]/\log(1/\delta)\leq T_{\rm gb}^{\Lambda}(F)$. For each $j \neq 1$ fix $\lambda_j \in \argmax_{\lambda\in\Lambda}\phi_j(\lambda)$, so that $\phi_j(\lambda_j)=\mathcal{G}^\Lambda_j(F)\geq g_\Lambda$. The variables $Y_{i,j}:=\log(1+\lambda_jM^{(1,j)}_i)$ are i.i.d. in $i$ with mean $\phi_j(\lambda_j)$ and, since $\lambda_j<1$ and $|M^{(1,j)}_i|\leq1$, bounded by $|Y_{i,j}|\leq B_\Lambda$. Let $L:=\log\big((K-1)|\Lambda|/\delta\big)$ and consider the good event
$$
E_n:=\bigcap_{j\neq1}\Big\{\sum_{i=1}^nY_{i,j}\geq L\Big\} \cap \big\{\hat k_n=1\big\}
$$
that at time $n$, the leader is the true modal answer and the cumulative log wealth for each challenger $j$ is greater than $L$. Condition on the good event $E_n$, we have for every $j\neq1$, $\overline W_n^{(1,j)}\geq|\Lambda|^{-1}\cdot W_n^{(1,j)}(\lambda_j)=|\Lambda|^{-1}\cdot\exp\big(\sum_{i\leq n}Y_{i,j}\big)\geq(K-1)/\delta$ (since the wealth is nonnegative), so the stopping condition holds at time $n$ with leader $a_1$ and $\tau^{\mathrm D}\leq n$. Therefore $\mathbb{P}_F(\tau^{\mathrm D}>n)\leq\mathbb{P}_F(E_n^c)$, and we bound the two parts of $E_n^c$ by Hoeffding's inequality: for i.i.d. $Z_1,\cdots,Z_n\in[a,b]$ with mean $\mu$ and $t>0$, we have
$$
\mathbb{P}\Big(\sum_{i=1}^nZ_i\leq n\cdot(\mu-t)\Big)\leq\exp\Big(-\dfrac{2nt^2}{(b-a)^2}\Big).
$$
First, $\hat k_n\neq1$ requires $S_{n,1}-S_{n,k}=\sum_{i\leq n}M^{(1,k)}_i\leq0$ for some $k \neq 1$, and $M^{(1,k)}_i\in[-1,1]$ has mean $p_1-p_k\geq\Delta$, so
$$
\mathbb{P}_F(\hat k_n\neq1)\leq\sum_{k\neq1}\mathbb{P}_F\Big(\sum_{i\leq n}M^{(1,k)}_i\leq0\Big)\leq(K-1)\cdot \exp(-n\Delta^2/2).
$$
Second, fix $\varepsilon\in(0,1)$ and let $n_0:=\lceil L/((1-\varepsilon)g_\Lambda)\rceil$. For $n\geq n_0$ we have $L\leq(1-\varepsilon)\cdot ng_\Lambda\leq n\phi_j(\lambda_j)-\varepsilon ng_\Lambda$, so $\sum_{i\leq n}Y_{i,j}<L$ implies $\sum_{i\leq n}Y_{i,j}\leq n\cdot(\phi_j(\lambda_j)-\varepsilon g_\Lambda)$, and Hoeffding's inequality with $[a,b]=[-B_\Lambda,B_\Lambda]$ and $t=\varepsilon g_\Lambda$ gives
$$
\mathbb{P}_F\Big(\sum_{i\leq n}Y_{i,j}<L\Big)\leq\exp\Big(-\dfrac{n\varepsilon^2g_\Lambda^2}{2B_\Lambda^2}\Big), \qquad n\geq n_0.
$$
With $\gamma:=\min\{\Delta^2/2,\varepsilon^2g_\Lambda^2/(2B_\Lambda^2)\}>0$ we obtain $\mathbb{P}_F(\tau^{\mathrm D}>n)\leq2(K-1)\cdot e^{-\gamma n}$ for all $n\geq n_0$. Hence $\tau^{\mathrm D}<\infty$ almost surely, and
\begin{align}\label{eqn:nonasymp}
\begin{split}
\mathbb{E}_F[\tau^{\mathrm D}]&=\sum_{n\geq0}\mathbb{P}_F(\tau^{\mathrm D}>n)
\leq n_0+\sum_{n\geq n_0}2(K-1)\cdot \exp(-\gamma n)\\
&\leq\dfrac{\log\big((K-1)|\Lambda|/\delta\big)}{(1-\varepsilon)\cdot g_\Lambda}+1+\dfrac{2(K-1)}{1-\exp(-\gamma)}.
\end{split}
\end{align}
For a fixed grid the last two terms do not depend on $\delta$, so dividing by $\log(1/\delta)$, letting $\delta\to0$ and then $\varepsilon\to0$ gives $\limsup_{\delta\to0}\mathbb{E}_F[\tau^{\mathrm D}]/\log(1/\delta)\leq T_{\rm gb}^{\Lambda}(F)$.

\emph{Part (ii).} We write $\tau^{\mathrm D}_\delta$ for the stopping time with the grid $\Lambda_\delta:=\Lambda_{r_\delta,m_\delta}$ and $\ell:=\log(1/\delta)$ and let $\delta$ be small enough so that $\ell \geq 1$. Now according to the definition, we have $r_{\delta} = 1 + \ell^{-1/4}$ and $m_\delta := \big\lceil \ell ^{1/2}\big\rceil$. Then $1<r_\delta\leq2$ and $m_\delta\leq2\ell^{1/2}$. Since $\log(1+x)\geq x/2$ for $x\in[0,1]$,
$$
m_\delta\log r_\delta \geq \ell^{1/2}\cdot\dfrac{1}{2} \ell^{-1/4} =\dfrac{1}{2}\ell^{1/4}\to\infty,
$$
so $r_\delta^{-m_\delta}\to 0$ and the condition $r_\delta^{-m_\delta}\leq\min_{j\neq 1} \lambda^\star_j$ of Part (i) holds for all $\delta$ small enough. For such $\delta$, \eqref{eqn:gridloss} gives $g_{\Lambda_\delta}\geq1/(r_\delta T_{\rm gb}(F))\geq1/(2T_{\rm gb}(F))$, and the argument of Part (i) gives $\tau^{\mathrm D}_\delta<\infty$ almost surely together with \eqref{eqn:nonasymp}.

It remains to control how the grid enters \eqref{eqn:nonasymp}. First, $|\Lambda_\delta|\leq2m_\delta\leq 4\ell^{1/2}$, so $\log((K-1)|\Lambda_\delta|)=O(\log\ell)$. Second, a point $\lambda = 1-r_\delta^{-i}$ contributes $\log(1/(1-\lambda))=i\log r_\delta\leq m_\delta\log r_\delta\leq m_\delta(r_\delta-1)\leq2\ell^{1/4}$ to $B_{\Lambda_\delta}$, and a point $\lambda = r_\delta^{-i}$ contributes at most $\log(1/(1-r_\delta^{-1}))=\log(1+1/(r_\delta-1))=\log(1+\ell^{1/4})\leq\ell^{1/4}$. Hence $B_{\Lambda_\delta}\leq 2\ell^{1/4}$ and now we need
$$
\gamma\geq\min\Big\{\dfrac{\Delta^2}{2}, \dfrac{\varepsilon^2}{32T_{\rm gb}(F)^2\,\ell^{1/2}}\Big\}.
$$
Using $1/(1-\exp(-\gamma))\leq1+1/\gamma$, the last two terms of \eqref{eqn:nonasymp} are $O(\ell^{1/2})$ for fixed $\varepsilon$. Dividing \eqref{eqn:nonasymp} by $\ell$ and using $1/g_{\Lambda_\delta}\leq r_\delta T_{\rm gb}(F)$ gives:
$$
\dfrac{\mathbb{E}_F[\tau^{\mathrm D}_\delta]}{\log(1/\delta)}\leq\dfrac{r_\delta\cdot T_{\rm gb}(F)}{1-\varepsilon}\cdot\Big(1+\dfrac{\log((K-1)|\Lambda_\delta|)}{\ell}\Big)+O(\ell^{-1/2})\xrightarrow[\delta\to0]{}\dfrac{T_{\rm gb}(F)}{1-\varepsilon}.
$$
since $r_\delta\to1$. Letting $\varepsilon \to 0$ then proves \eqref{eqn:ascd-exact}.
\end{proof}

\newpage
\section{A Real-World Grey-Box Observation Example}\label{app:greybox-demo}

We give a complete example from the released demo data to make the grey-box observation model operationally concrete. The record is trajectory 0 (seed 20260922) for ARC-Challenge test item \texttt{MCAS\_2001\_5\_3}, generated by \texttt{mlx-community/Llama-3.2-3B-Instruct-bf16} at reasoning temperature $T_r=0.6$ (top-$p=0.95$, top-$k=20$). Answer probabilities are evaluated at the natural temperature $T_a=1$.

\paragraph{Original problem}
\begin{quote}\small
\textbf{Question:} Which of the following is LEAST responsible for the weathering of rocks?\\
\textbf{A.} freezing\\
\textbf{B.} plant growth\\
\textbf{C.} rain\\
\textbf{D.} lightning
\end{quote}
The benchmark answer is D.

\paragraph{Prompt}
The model receives the following system and user messages through its chat template.
\begin{quote}\small
\textbf{System:} You are a careful reasoner. The final answer must be exactly one of A, B, C, or D.

\medskip
\textbf{User:} Solve this multiple-choice problem carefully.

\medskip
Question:\\
Which of the following is LEAST responsible for the weathering of rocks?

\medskip
Options:\\
A. freezing\\
B. plant growth\\
C. rain\\
D. lightning

\medskip
Give a concise explanation in at most 120 words. Do not repeat the question or options. During this stage, do not output the final A/B/C/D letter and do not quote the exact chosen option. End exactly with \texttt{<END\_REASONING>}.
\end{quote}

\paragraph{LLM output}
The sampled reasoning trajectory, before answer scoring, is:
\begin{quote}\small
Plant growth is not a primary cause of weathering, as it primarily contributes to physical weathering through roots breaking apart the rock. Freezing, on the other hand, is a significant factor in chemical weathering, causing ice to expand and exert pressure on the rock, leading to cracking. Rain is a major contributor to physical weathering, as it can physically wear away the rock surface through erosion and abrasion. Lightning, while not directly causing weathering, can contribute to it by igniting wildfires that can further break down the rock.

\medskip
\texttt{END\_REASONING}
\end{quote}

\paragraph{The observed probability vector}
After this trajectory, we append the decision prefix
\begin{quote}\small\ttfamily
Final answer (one letter only):
\end{quote}
The prefix ends in the token "\texttt{):}" (token ID 1680). Each allowed continuation is one token: \texttt{ A} (362), \texttt{ B} (426), \texttt{ C} (356), or \texttt{ D} (423). Thus the observation comes from one precisely identified next-token position: the transitions from the terminal \texttt{):} token to each of the four option tokens. If $\ell_k$ is the full-vocabulary log-probability of option token $k$ at that position, we record
\[
q_k=\exp(\ell_k),
\qquad
\theta_k=\frac{\exp(\ell_k/T_a)}{\sum_{j\in\{\mathrm A,\mathrm B,\mathrm C,\mathrm D\}}\exp(\ell_j/T_a)}.
\]
For this trajectory, the exact values are:

\begin{table}[H]

    \caption{A single real-world grey-box observation. The raw probabilities $q_k$ are the full-vocabulary next-token probabilities assigned to the four allowed label tokens. Their total mass is $0.9630$ and renormalizing over A--D at $T_a=1$ gives the observed vector $\boldsymbol{\theta}=(0.2635,0.4344,0.2052,0.0969)$.}
    \label{tab:greybox-demo}
    \begin{center}
    \small
    \begin{tabular}{lrrrr}
        \toprule
        Option token & Token ID & $\ell_k$ & $q_k$ & $\theta_k$ \\
        \midrule
        \texttt{ A} & 362 & $-1.3715$ & $0.2537$ & $0.2635$ \\
        \texttt{ B} & 426 & $-0.8715$ & $0.4183$ & $0.4344$ \\
        \texttt{ C} & 356 & $-1.6215$ & $0.1976$ & $0.2052$ \\
        \texttt{ D} & 423 & $-2.3715$ & $0.0933$ & $0.0969$ \\
        \midrule
        Sum & & & $0.9630$ & $1.0000$ \\
        \bottomrule
    \end{tabular}
    \end{center}
\end{table}

This observation is not a product of probabilities along the reasoning trajectory. It is a single next-token softmax, restricted and renormalized over the known candidate labels, after the sampled reasoning has been placed in the context.

\paragraph{Final answer}
Since $\argmax_k\theta_k=\mathrm B$, greedy decoding returns \textbf{B}. The released record also stores an independently seeded paired black-box draw $o\sim\Cat(\boldsymbol{\theta})$, which happened to be A in this trajectory; ASC-D instead observes the entire vector in Table~\ref{tab:greybox-demo}. The greedy answer B and the paired draw A are both incorrect relative to the benchmark answer D, illustrating that the grey-box observation exposes uncertainty rather than assuming that a single trajectory is correct.

\newpage

\section{Additional Synthetic Experiments}\label{app:synthetic-additional}

The following experiments test two distinct theoretical predictions. The first validates the characteristic-time scaling for a general, non-Dirichlet feedback law, while the second examines how the gain from distribution-valued feedback changes with Dirichlet concentration.

\paragraph{Characteristic Time Validation}
To test the characteristic-time prediction beyond the Dirichlet model, we fix $K=10$ and $\boldsymbol{\pi}=(0.4,0.3,0.0375\mathbf 1_8)$ and consider the discrete family
\begin{align}\label{eqn:soft-vertex}
Z\sim\Cat(\boldsymbol{\pi}), \qquad
\boldsymbol{\theta}=(1-\rho)\boldsymbol{e}_Z+\rho\boldsymbol{\pi},
\qquad \rho\in[0,1].
\end{align}
Every $F_\rho$ has mean $\boldsymbol{\pi}$ and therefore induces the same black-box law $\Cat(\boldsymbol{\pi})$. Increasing $\rho$ changes only the information in the probability vector, moving from the vertex law $F_0$ to the degenerate law $F_1=\delta_{\boldsymbol{\pi}}$.
For $\rho\in\{0,0.2,0.4,0.6,0.8,1\}$, we compute $T_{\rm gb}(F_\rho)$ and run 2,000 repetitions of ASC-D for $\delta\in\{0.2,0.1,0.05,0.02,0.01,0.002,10^{-3},2\times10^{-4},10^{-5},10^{-6},10^{-8}\}$. We use $\Lambda_{1.5,16}$ and a maximum budget of 20,000. For each $\rho$, we estimate the empirical characteristic time by fitting $\mathbb{E}[\tau^{\rm D}]=\alpha_\rho+\beta_\rho\log(1/\delta)$ over $\delta\leq10^{-3}$. Confidence intervals use 2,000 paired bootstrap samples.

\begin{figure}[H]
    \centering
    \includegraphics[width=0.8\linewidth]{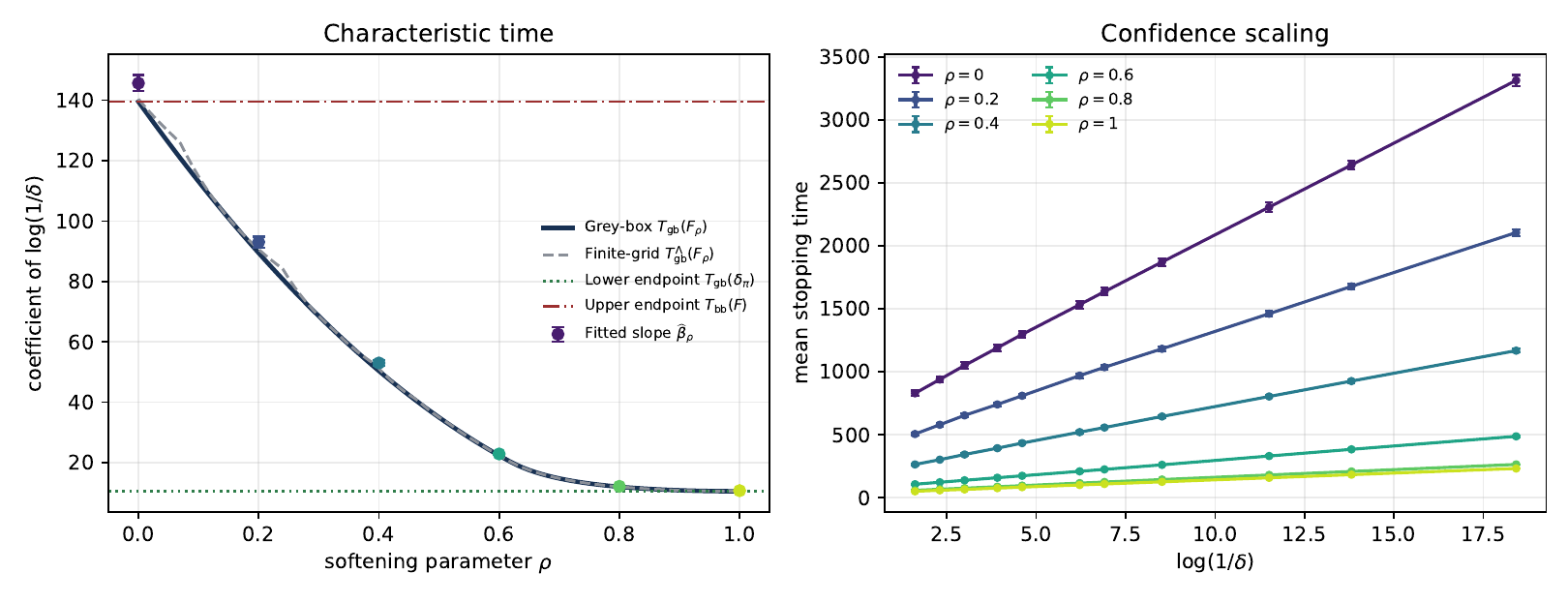}
    \caption{Characteristic-time scaling across the feedback family $F_\rho$. }
    \label{fig:synthetic-general}
\end{figure}

Figure~\ref{fig:synthetic-general} compares theoretical and empirical characteristic times (left) and mean stopping time versus $\log(1/\delta)$ (right). Although all $F_\rho$ induce the same black-box law, $T_{\rm gb}(F_\rho)$ decreases from the vertex upper bound toward the degenerate lower bound as $\rho$ increases. The fitted slopes differ from $T_{\rm gb}(F_\rho)$ by at most $5.2\%$, while the finite grid increases the predicted leading constant by at most $1.53\%$. The approximately linear trends in the right panel support the predicted logarithmic dependence on $1/\delta$.

\paragraph{Effect of Dirichlet Concentration}

To examine the prediction in Proposition~\ref{prop:hard}, we fix $K=10$, $\delta=0.05$, and $\boldsymbol{\pi}=(0.4,0.3,0.0375\mathbf 1_8)$, and draw $\boldsymbol{\theta}_t\sim\operatorname{Dir}(c\boldsymbol{\pi})$ for $c\in\{0.01,0.1,1,3,10,100,1000\}$. For each $c$, we run 1,000 paired trials of ASC-D using either the full vector $\boldsymbol{\theta}_t$ or the one-hot sample $\boldsymbol e_{Y_t}$, where $Y_t\mid\boldsymbol{\theta}_t\sim\operatorname{Cat}(\boldsymbol{\theta}_t)$. Both versions use the same grid, stopping rule, and maximum budget of 65,536.

\begin{figure}[H]
    \centering
    \includegraphics[width=0.5\linewidth]{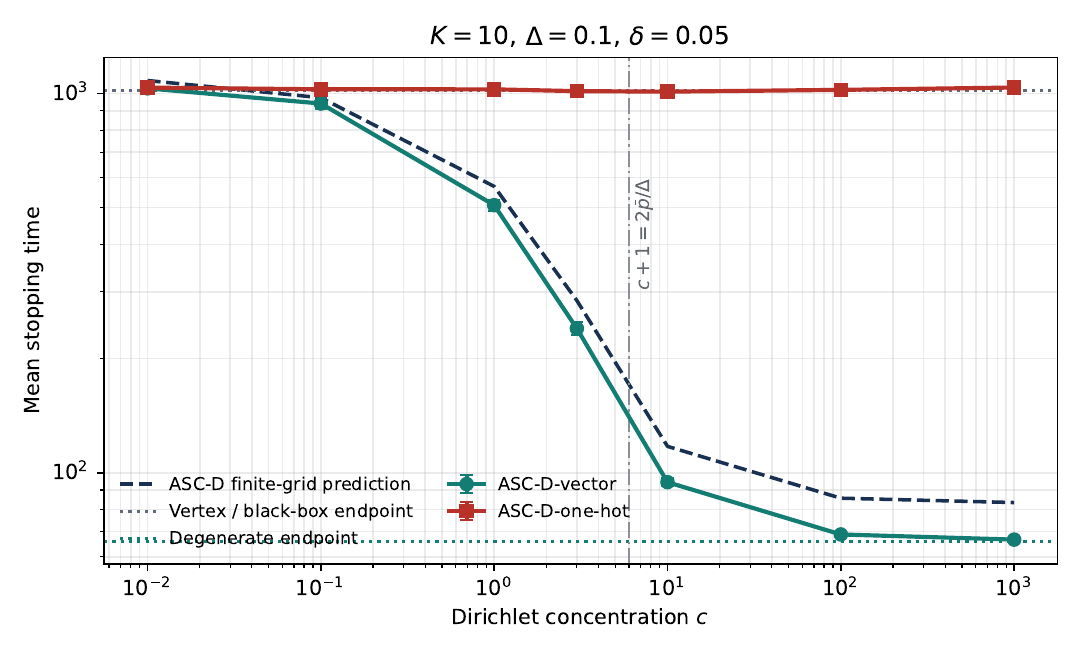}
    \caption{Mean stopping time versus Dirichlet concentration.}
    \label{fig:synthetic-concentration}
\end{figure}

As shown in Figure~\ref{fig:synthetic-concentration}, the one-hot stopping time is nearly constant because its marginal feedback law remains $\Cat(\boldsymbol{\pi})$, whereas full-vector ASC-D becomes faster as increasing concentration makes trajectories more consistent. The two methods perform similarly near the vertex endpoint; full-vector ASC-D is $4.22\times$ faster at $c=3$ and $15.56\times$ faster at $c=1000$, close to the limiting gain of $15.58\times$. The concentration-dependent gain and its eventual saturation are consistent with Proposition~\ref{prop:hard}.

\newpage

\section{Additional Real-World Analyses}\label{app:real-world-additional}

This section supplements Section~\ref{sec:real-world} with uncertainty estimates for the modal-gap scaling analysis and budget-resolved results for the fixed-budget cross-dataset replication.

\paragraph{Uncertainty in Modal-Gap Scaling}
For each of the 90 common questions in Exp.~1, we regress log required budget on $\log(1/\Delta)$ across $T_a\in\{4,8,16,32\}$. Table~\ref{tab:exp1-gap-scaling} reports the median within-question slope. Here the confidence intervals use 5,000 paired question-bootstrap resamples.

\begin{table}[H]
    \setlength{\tabcolsep}{4pt}
    \caption{Gap-scaling exponents on the 90 common questions.}
    \begin{center}
    \label{tab:exp1-gap-scaling}
    \begin{tabular}{lcccc}
        \toprule
        & ASC-D & ASC-D (one-hot) & PPR-1v1 & CITE-4 \\
        \midrule
        Median $\widehat{\beta}$ & \textbf{0.766} & 1.969 & 1.966 & 1.854 \\
        95\% CI & $[0.692,0.810]$ & $[1.953,2.001]$ & $[1.949,1.990]$ & $[1.841,1.861]$ \\
        \bottomrule
    \end{tabular}
    \end{center}
\end{table}

\begin{figure}[H]
    \centering
    \includegraphics[width=0.7\linewidth]{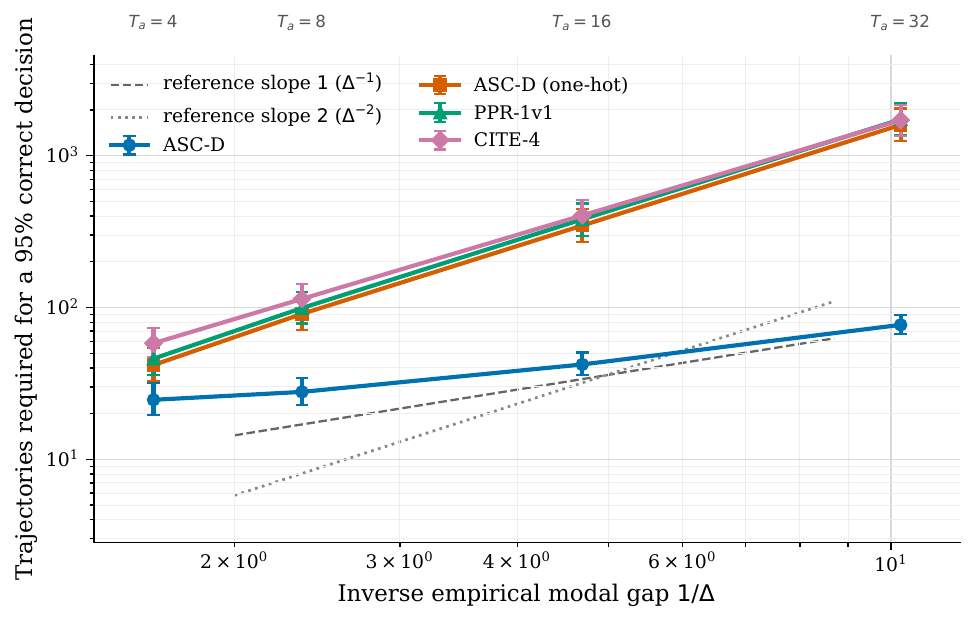}
    \caption{Required trajectory budget versus inverse empirical modal gap on the 90 common questions. Points show geometric means, error bars show 95\% paired question-bootstrap confidence intervals, and reference lines have slopes one and two on the log--log scale.}
    \label{fig:exp1-required-budget-loglog}
\end{figure}
The bootstrap intervals separate ASC-D from its matched one-hot ablation and the other hard-label methods. Figure~\ref{fig:exp1-required-budget-loglog} shows the same contrast directly: as the modal gap narrows, the required budget grows much more slowly for ASC-D than for ASC-D (one-hot), PPR-1v1, and CITE-4. These are finite-range empirical slopes and do not imply asymptotic scaling better than $\Delta^{-1}$.

\paragraph{Cross-Dataset Replication on ARC-Challenge}
We replicate Exp.~2 on a fixed set of 100 four-choice ARC-Challenge questions using Qwen3-1.7B and Llama-3.2-1B. For each model and question, 128 probability vectors define the empirical population and its modal answer; we retain all 100 questions without screening by modal majority. We run 200 paired replays at $1-\delta\in\{0.95,0.975,0.99\}$ with the same $N=128$ trajectory budget as in Section~\ref{sec:real-world}. At each step, ASC-D observes a resampled probability vector, whereas CITE-4 and PPR-1v1 observe the same paired categorical draw. As in Exp.~2, a replay succeeds only if the method stops within the budget and returns the empirical mode.

\begin{table}[H]
    \caption{Correct-certification rate within $N=128$ trajectories on 100 ARC-Challenge questions. Each entry aggregates 200 paired replays per question.}
    \label{tab:arc-cross-dataset-certification}
    \begin{center}
    \small
    \setlength{\tabcolsep}{7pt}
    \renewcommand{\arraystretch}{1.08}
    \begin{tabular}{llccc}
        \toprule
        Model & Method
        & $1-\delta=0.95$
        & $1-\delta=0.975$
        & $1-\delta=0.99$ \\
        \midrule
        Llama-3.2-1B
            & ASC-D   & \textbf{74.48\%} & \textbf{71.99\%} & \textbf{69.06\%} \\
            & CITE-4  & 55.59\% & 53.40\% & 51.20\% \\
            & PPR-1v1 & 54.25\% & 52.43\% & 50.30\% \\
        \midrule
        Qwen3-1.7B
            & ASC-D   & \textbf{97.48\%} & \textbf{97.27\%} & \textbf{96.94\%} \\
            & CITE-4  & 97.07\% & 96.72\% & 96.20\% \\
            & PPR-1v1 & 96.86\% & 96.55\% & 95.97\% \\
        \bottomrule
    \end{tabular}
    \end{center}
\end{table}
Table~\ref{tab:arc-cross-dataset-certification} shows that ASC-D achieves the highest correct-certification rate for both models at every confidence level. The gains are larger for Llama-3.2-1B: ASC-D outperforms CITE-4 by $17.86$--$18.89$ percentage points and PPR-1v1 by $18.76$--$20.23$ percentage points. For Qwen3-1.7B, where all methods already achieve high certification rates, the gains are smaller but remain consistent, ranging from $0.41$ to $0.97$ percentage points. These results demonstrate that the benefit of distribution-valued feedback transfers from MMLU-Redux to ARC-Challenge and is especially substantial in more challenging certification settings.

\end{document}